\documentclass[11pt,letter]{article}
\usepackage[papersize={8.5in,11in},margin=1in]{geometry}

\usepackage[utf8]{inputenc} % allow utf-8 input
\usepackage[T1]{fontenc}    % use 8-bit T1 fonts
\usepackage{hyperref}       % hyperlinks
\usepackage{url}            % simple URL typesetting
\usepackage{booktabs}       % professional-quality tables
\usepackage{amsfonts}       % blackboard math symbols
\usepackage{nicefrac}       % compact symbols for 1/2, etc.
\usepackage{microtype}      % microtypography
\usepackage{pifont}% http://ctan.org/pkg/pifont

\usepackage[usenames, dvipsnames]{color}

\usepackage{amssymb}
\usepackage{graphicx}

\usepackage{times}
\usepackage{helvet}
\usepackage{courier}
\newtheorem{theorem}{Theorem}
\newtheorem{assumption}{Assumption}
\newtheorem{proof}{Proof}

\newtheorem{corollary}{Corollary}
\newtheorem{lemma}{Lemma}
\usepackage{algorithm}
\usepackage{algcompatible}
\usepackage{amsmath}
\usepackage{subcaption}
\usepackage{booktabs}
\usepackage{epsfig}
\usepackage{epstopdf}
\usepackage{url}
\usepackage{multirow}
\newcommand*{\QEDB}{\hfill\ensuremath{\square}}%

\title{Distributed Asynchronous Dual Free Stochastic Dual Coordinate Ascent}

\begin{document}

\author{Zhouyuan Huo 
	\and 
	Heng Huang
}

\maketitle

\begin{abstract}
	The primal-dual distributed optimization methods have broad large-scale machine learning applications. Previous primal-dual distributed methods are not applicable when the dual formulation is not available, \emph{e.g.} the sum-of-non-convex objectives. Moreover, these algorithms and theoretical analysis are based on the fundamental assumption that the computing speeds of multiple machines in a cluster are similar.  However, the straggler problem is an unavoidable practical issue in the distributed system because of the existence of slow machines. Therefore, the total computational time of the distributed optimization methods is highly dependent on the slowest machine. In this paper, we address these two issues by proposing distributed asynchronous dual free stochastic dual coordinate ascent algorithm for distributed optimization. Our method does not need the dual formulation of the target problem in the optimization. We tackle the straggler problem through asynchronous communication and the negative effect of slow machines is significantly alleviated. We also analyze the convergence rate of our method and prove the linear convergence rate even if the individual functions in objective are non-convex. Experiments on both convex and non-convex loss functions are used to validate our statements.
\end{abstract}

\section{Introduction}
In this paper, we consider optimizing the $\ell_2$-norm regularized empirical loss minimization problem which is arising ubiquitously in supervised machine learning:
\begin{eqnarray}
	\label{primal}
	\min\limits_{w \in \mathbb{R}^d} P(w )  := 	\min\limits_{w \in \mathbb{R}^d} {\frac{1}{n} \sum\limits_{i=1}^n \phi_i( w) }+  \frac{\lambda}{2} \|w\|^2_2.
\end{eqnarray}
We let $f(w) = {\frac{1}{n} \sum_{i=1}^n \phi_i( w) }$ and  $w \in \mathbb{R}^d$ be the linear predictor to be optimized. There are many applications falling into this formulation, such as classification, regression, and principal component analysis (PCA).  In classification, given features $x_i \in \mathbb{R}^d$ and labels  $y_i \in \{1, -1\}$, we obtain Support Vector Machine (SVM) when we let $\phi_i( w) = \max \{ 0, 1-  y_i x_i^T w\}$. In regression, given features $x_i \in \mathbb{R}^d$ and response $y_i \in \mathbb{R} $, we have Ridge Regression problem if $\phi_i( w) = (y_i - x_i^T w)^2$.
Recently, \cite{garber2015fast,allen2016improved} showed that the problem of PCA can be solved through convex optimization. Supposing $C = \frac{1}{n} \sum_{i=1}^n x_i x_i^T$ be normalized covariance matrix,  \cite{garber2015fast} showed that approximating the principle component of $A$ is equivalent to minimizing $f(w) =  \frac{1}{2} w^T(\mu I - C)w- b^Tw$ given $\mu >0$ and $b\in \mathbb{R}^d$. Defining $\phi_i(w) =\frac{1}{2} w^T((\mu -\frac{\lambda}{2}) I - x_ix_i^T)w - b^Tw$ and $\mu > \sigma_1(C) + \frac{\lambda}{2}$ where $\sigma_1(C)$ denotes the largest singular value of $C$, it also falls into  problem (\ref{primal}). In this case, $f(w)$ is convex while each $\phi_i(w)$ is probably non-convex.

Distributed machine learning methods are required to optimize problem (\ref{primal}) when the dataset is distributed over multiple machines.
In \cite{jaggi2014communication}, the authors proposed communication-efficient distributed dual coordinate ascent (CoCoA) for primal-dual distributed optimization. In each iteration, the CoCoA framework allows workers to optimize subproblems independently at first. After that, it calls ``Reduce" operation to collect local solution from all workers, and updates global variable and broadcasts the up-to-date global variable to workers in the end. It uses stochastic dual coordinate ascent (SDCA) as the local solver which is one of the most successful methods proposed for solving the problem (\ref{primal}) \cite{hsieh2008dual,shalev2013accelerated}. In \cite{shalev2013stochastic}, the authors proved that SDCA has linear convergence if the convex function $\phi_i(w)$ is smooth, which is much faster than stochastic gradient descent (SGD).   \cite{yang2013trading, takavc2015distributed} also proposed distributed SDCA and analyzed the tradeoff between computation and communication.
\cite{ma2015adding,ma2017distributed} accelerated the CoCoA by allowing for more aggressive updates, and proved that CoCoA has linear primal-dual convergence for the smooth convex problem and sublinear convergence for the non-smooth convex problem. However, there are two issues for these primal-dual distributed methods. Firstly, all of them use SDCA as the local solver. SDCA is not applicable when the dual problem is unknown, \emph{e.g.} $\phi_i$ is non-convex. Therefore, the applications of these primal-dual distributed methods are limited. Secondly, all of these methods assume that the workers have similar computing speed, which is not true in practice. Straggler problem is an unavoidable practical issue in the distributed optimization. Thus, the computing time of CoCoA and distributed SDCA is dependent on the slowest worker. Even if there is only one bad worker, they will work far slower than expectation.

In \cite{shalev2015sdca, shalev2016sdca}, the authors proposed dual free stochastic dual coordinate ascent (dfSDCA). It was proved to admit similar convergence rate to SDCA while it did not rely on duality at all. However there is no distributed optimization method using dfSDCA, and its convergence analysis is still unknown yet.

In this paper, we solve the above two challenging issues in previous primal-dual distributed optimization methods by proposing Distributed Dual Free Stochastic Dual Coordinate Ascent (Dis-dfSDCA).  We use dfSDCA as the local solver such that Dis-dfSDCA can be applied to the non-convex problem easily. We alleviate the effect of straggler problem by allowing asynchronous communication between server and workers. As shown in Figure \ref{intro}, the server does not wait and workers may store the stale global variable in the local. We also analyze the convergence rate of our method and prove that it admits linear convergence rate even if the individual losses  ($\phi_i$) are non-convex, as long as the sum of losses $f$ is convex. Finally, we conduct simulation on the distributed system with straggler problem. Experimental results verify our theoretical conclusions and show that our method works well in practice.

\begin{figure}[t]
	\centering
	\includegraphics[width=3.2in]{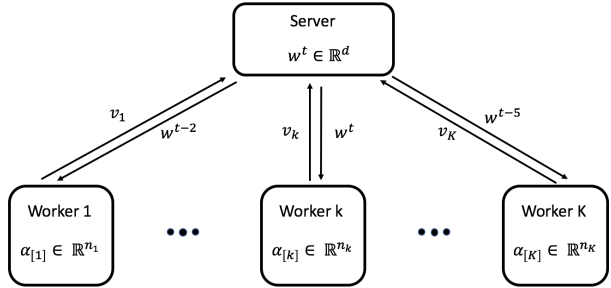}
	\caption{Distributed asynchronous dual free stochastic dual coordinate ascent for parameter server framework. In iteration $t$, the server receives gradient message $v_k$ from worker $k$, and sends the up-to-date $w^{t}$ back to the worker $k$. Global variables in other workers are stale. For example worker $1$ and $K$ store stale global variables $w^{t-2}$ and $w^{t-5}$ respectively. }
	\label{intro}
\end{figure}

\section{Preliminary}
To optimize the primal problem (\ref{primal}), we often derive and optimize its dual problem alternatively:
\begin{eqnarray}
	\max\limits_{\alpha \in \mathbb{R}^n} D(\alpha):=	\max \limits_{\alpha \in \mathbb{R}^n} \frac{1}{n} \sum\limits_{i=1}^n - \phi_i^*(-\alpha_i) - \frac{\lambda}{2} \|\frac{1}{\lambda n} A\alpha \|^2_2\,,
	\label{dual}
\end{eqnarray}
where $\phi_i^*$ is the convex conjugate function to $\phi_i$, $A  = [x_1,x_2,...x_n] \in \mathbb{R}^{d\times n}$ denotes data matrix and $\alpha \in \mathbb{R}^n$ denotes dual variable. We can use stochastic gradient descent (SGD) to optimize primal problem (\ref{primal}), however, there are always two issues: (1) SGD is too aggressive at the beginning of the optimization; (2) it does not have a clear stopping criterion. One of the biggest advantages of optimizing the dual problem is that we can keep tracking the duality gap $G(\alpha)$ to monitor the progress of optimization. Duality gap is defined as: $G(\alpha) = P(w(\alpha)) - D(\alpha)$, where $P(w(\alpha))$ and $D(\alpha)$ denote objective values of primal problem and dual problem respectively. If $w^*$ is the optimal solution of primal problem (\ref{primal}) and $\alpha^*$ is the optimal solution of dual problem (\ref{dual}), the primal-dual relation always holds that:
\begin{eqnarray}
	\label{primal_dual}
	w^* = w(\alpha^*) = \frac{1}{\lambda n} A \alpha^*\,.
\end{eqnarray}

\subsection{Stochastic Dual Coordinate Ascent}
In \cite{shalev2013stochastic}, the authors proposed stochastic dual coordinate ascent (SDCA) to optimize the dual problem (\ref{dual}). The pseudocode of SDCA is presented in Algorithm \ref{sdca}. In iteration $t$, given sample $i$ and other dual variables $\alpha_{j\neq i}$  fixed, we maximize the following subproblem:
\begin{eqnarray}
	\label{dual_sub}
	\max \limits_{\Delta \alpha_i \in \mathbb{R}} -\frac{1}{n}  \phi_i^*(-(\alpha_i^{t} + \Delta \alpha_i )) - \frac{\lambda}{2} \| w^{t} + \frac{1}{\lambda n}\Delta \alpha_i x_i \|^2_2
\end{eqnarray}
$e_i$ denotes coordinate vector of size $n$, where element $i$ is $1$ and other elements  are $0$. In their paper, the authors proved that SDCA admits linear convergence rate for smooth loss, which is much faster than stochastic gradient descent (SGD).  An accelerated SDCA was also proposed in \cite{shalev2013accelerated}. However, SDCA is not applicable when it is difficult to derive the dual problem, \emph{e.g.} $\phi_i$ are non-convex.

\begin{algorithm}[h]
	\caption{SDCA}
	\begin{algorithmic}[1]
		\STATE Initialize $\alpha^0$ and $w^0 = w(\alpha^0)$;
		\FOR{$t=0,1,2,\dots, T-1$}
		\STATE Randomly sample $i$ from $\{1,2,...,n\}$;
		\STATE Find $\Delta \alpha_i $ to maximize the subproblem (\ref{dual_sub});
		\STATE Update dual variable $\alpha$ through: \\
		\hspace{1.5cm} $\alpha^{t+1}  \leftarrow \alpha^{t} + \Delta \alpha_i e_i  $;
		\STATE Update primal variable $w$ through:\\
		\hspace{1.5cm} $w^{t+1} \leftarrow w^{t} + \frac{1}{\lambda n} \Delta \alpha_i x_i$;
		\ENDFOR
	\end{algorithmic}
	\label{sdca}
\end{algorithm}

\subsection{Dual Free Stochastic Dual Coordinate Ascent}
To address the limitation of SDCA,  \cite{shalev2015sdca} proposes Dual Free Stochastic Coordinate Ascent (dfSDCA) which has similar convergence property to SDCA. The pseudocode of dfSDCA is presented in Algorithm \ref{dfsdca}.
Although we keep vector $\alpha \in \mathbb{R}^n$ in the optimization, the derivation of dual problem is not necessary for dfSDCA. According to the update rule of $\alpha$ and $w$ in the algorithm, the primal-dual relation (\ref{primal_dual}) also holds for dfSDCA. The drawback of dfSDCA is that it is space-consuming to store $\alpha$, whose space complexity $O(nd)$. We can reduce it to $O(n)$ if $\nabla \phi_i(w)$ can be written as $\nabla \phi_i(x_i^Tw)x_i$. In \cite{he2015dual}, the authors accelerated dfSDCA by using non-uniform sampling strategy in each iteration and proved that it admits faster convergence.

\begin{algorithm}[h]
	\caption{Dual Free SDCA}
	\begin{algorithmic}[1]
		\STATE Initialize dual variable $\alpha^0 = (\alpha_0^0, ..., \alpha_n^0)$ where $\forall i, \alpha_i^0 \in \mathbb{R}^d$,  primal variable $w^0 = w(\alpha^0)$;
		\FOR{$t=0,1,2,\dots, T-1$}
		\STATE Randomly sample $i$ from $\{1,2,...,n\}$;
		\STATE Compute dual residue $\kappa$ through: \\
		\hspace{1.5cm}$ \kappa  \leftarrow \nabla \phi_i(w^{ t}) + \alpha_i^t$;
		\STATE Update  dual variable $ \alpha_i$ through:\\
		\hspace{1.5cm} $\alpha_i^{t+1} \leftarrow \alpha_i^t - \eta \lambda   n \kappa$;
		\STATE Update primal variable $w$ through:\\
		\hspace{1.5cm} $w^{ t+1} \leftarrow w^{t} - \eta \kappa$;
		\ENDFOR
	\end{algorithmic}
	\label{dfsdca}
\end{algorithm}

\section{Distributed Asynchronous Dual Free Stochastic Dual Coordinate Ascent}
In this section, we propose Distributed Asynchronous Dual Free Stochastic Coordinate Ascent (Dis-dfSDCA) for distributed optimization. Dis-dfSDCA fits for any parameter server framework, where the star-shape network is used. We assume that there are $n$ samples in the dataset, and they are evenly distributed over $K$ workers. In  worker $k$, there are $n_k$ samples.  It is satisfied that $n = \sum_{k=1}^K n_k $. Different from  sequential dfSDCA, we split the update of dual variable and primal variable into different nodes. The pseudocodes of Dis-dfSDCA for server node and worker nodes are presented in Algorithm \ref{alg2} and Algorithm \ref{alg1} respectively.

\subsection{Update Global Variable $w$ on Server}
The up-to-date global variable $w \in \mathbb{R}^d$ is stored and updated on the server. Initially, $w$ is set to be vector zero. At the beginning of each iteration, the server receives gradient message $v_k$ from arbitrary worker $k$ and let $v^t = v_k$. Then it updates the global variable through:
\begin{eqnarray}
	w^{s, t+1} = w^{s,t} - \eta v^t
\end{eqnarray}
Finally, it sends the up-to-date global variable back to the worker $k$ for further computation.  Asynchronous method is robust to straggler problem because it allows for updating the global variable when receiving from only one worker. However, if the $w$ in the worker is too stale, it may lead the algorithm to diverge.  Therefore, we induce two loops in our algorithm. Server broadcasts the latest global variable $w$ to all workers after every $T$ iterations. In this way, we prevent the problem of divergence and keep the advantage of asynchronous communication at the same time.
Algorithm \ref{alg2} summarizes the pseudocode on the server.
\begin{algorithm}[h]                      % enter the algorithm environment
	\caption{Dis-dfSDCA (Server) }         % give the algorithm a caption
	\label{alg2}                           % and a label for \ref{} commands later in the document
	\begin{algorithmic}                    % enter the algorithmic environment
		\STATE Initialize $w \in \mathbb{R}^d$, $ \eta$
		\FOR{$s= 0, 1,..., S-1 $ }
		\FOR{$t = 0, 1,...,T-1$}
		\STATE	Receive gradient message $v^{s,t} = v_k$ from worker $k$;
		\STATE Update global variable $w^{s+1, t+1}$ through:\\
		\hspace{1.5cm} $w^{s, t+1} \leftarrow w^{s,t} - \eta v^{s,t }$;
		\STATE Send $w^{s, t+1}$ back to worker $k$ ;
		\ENDFOR
		\STATE $w^{s+1, 0} = w^{s, T}$
		\STATE Broadcast the up-to-date global variable $w^{s+1,0} $ to all workers.
		\ENDFOR
	\end{algorithmic}
\end{algorithm}

In Algorithm \ref{alg2}, we use the update of vanilla dfSDCA in the server. \cite{shalev2016sdca} proposed accelerated dfSDCA by using ``Catalyst" algorithm of \cite{lin2015universal}. It is proved to admit faster convergence rate by a constant factor. Our Algorithm \ref{alg2} can also be extended to the accelerated version easily. In our paper, we only consider the vanilla version and analyze the convergence rate of our algorithm.

\subsection{Update Local Variable $\alpha$ on  Worker}
In the distributed optimization, workers are responsible for the gradient computation which is the main workload during the optimization. We take arbitrary worker $k$ as an example.
Dual variable $\alpha_{[k]} \in \mathbb{R}^{n_k}$ is only stored and updated in the worker $k$, each $\alpha_i$ is corresponding to sample $i$. Initially, local variable $\alpha_{[k]}$ is set to be vector zero. After receiving stale global variable $w^{s, d(t)} \in \mathbb{R}^{d}$ from the server, worker $k$ computes the dual residue $\kappa$ and updates local variable $\alpha_i$ and gradient message $v_k$ for $H$ iterations.
Samples $I_t$ are randomly selected in the local dataset, and we set $|I_t|=H$. 	
\begin{algorithm}[h]                      % enter the algorithm environment
	\caption{Dis-dfSDCA (Worker $k$)}         % give the algorithm a caption
	\label{alg1}                           % and a label for \ref{} commands later in the document
	\begin{algorithmic}
		\STATE Initialize $\alpha_{[k]} \in \mathbb{R}^{d \times n_k }$, $\eta$, $H$ 
		\REPEAT
		% enter the algorithmic environment
		\STATE Receive global variable $w^{s, d(t)}$ from server;
		\STATE	Initialize gradient message: $v_k \leftarrow {0}$;
		\STATE Randomly select samples $I_t$ from $\{1,\cdots,n_k \}$ where $|I_t|=H$;
		\FOR{sample $i$ in $I_t$}
		\STATE Compute dual residue $\kappa$ through: \\
		\hspace{1.5cm}$ \kappa  \leftarrow \nabla \phi_i(w^{s, d(t)}) + \alpha_i$;
		\STATE Update  local dual variable $ \alpha_i$ through:\\
		\hspace{1.5cm} $\alpha_i \leftarrow \alpha_i - \eta \lambda   n \kappa $;
		\STATE Update  gradient message $v_k$ through:\\
		\hspace{1.5cm} $ v_k \leftarrow   v_k + \kappa$;
		\ENDFOR
		\STATE Send  gradient message $v_k$ to server;
		\UNTIL{Termination}
	\end{algorithmic}
\end{algorithm}
In each iteration, worker $k$ selects a sample $i$ randomly and computes the dual residue  $\kappa$ for coordinate $i$ of the dual variable through:
\begin{eqnarray}
	\kappa = \nabla \phi_i(w^{s, d(t)}) + \alpha_i
\end{eqnarray}
Dual residue can also be viewed as the gradient in Stochastic Gradient Descent. When we obtain optimal dual variable $\alpha^*$  and primal variable $w^*$, $\kappa$ should  be $0$. Therefore, it is satisfied that $\alpha_i^*=-\nabla \phi_i(w^*)$.
Then worker $k$ updates local dual variable $\alpha_i$ and gradient message $v_k$ separately through:
\begin{eqnarray}
	\alpha_i& = &\alpha_i- \eta \lambda n \kappa, \hspace{0.3cm}  i \in I_t   \\
	v_k & =& v_k + \kappa
\end{eqnarray}
Because there is only one $\alpha_i$ in the cluster, it is always up-to-date.
After $H$ iterations, the worker $k$ sends gradient message $v_k$ to the server.   From the update rule in our algorithm, it is easy to know that the well-known primal-dual relation in the equation (\ref{primal_dual}) is always satisfied.
The pseudocode of Dis-dfSDCA in worker node $k$ is described in Algorithm \ref{alg1}. 	

In Algorithm \ref{alg1},  we use vanilla dfSDCA in the worker which samples with uniform distribution.  There are also other sampling techniques proposed to accelerate dfSDCA. As per the sampling strategy in \cite{shalev2015sdca,he2015dual,shalev2016sdca,qu2015stochastic}, there are three options: uniform sampling, importance sampling, and adaptive sampling. In importance sampling strategy \cite{shalev2016sdca}, it first computes the fixed probability distribution $p_i$ using smoothness parameter of each function $\phi_i$, then selects samples following this probability. In adaptive sampling strategy \cite{he2015dual}, it computes the adaptive probability distribution $p_i$ using dual residue $\kappa$ for each sample every iteration, then selects samples following this probability. Both of them are proved to admit faster convergence than vanilla dfSDCA with uniform sampling.  We only consider the uniform sampling strategy, and analyze its corresponding convergence rate in our paper. However, other sampling techniques are straightforward to be applied to our distributed method.

\section{Convergence Analysis}
In this section, we provide the theoretical convergence analysis of Dis-dfSDCA. For the case of convex losses $\phi_i$, we prove that Dis-dfSDCA admits linear convergence rate.
If losses $\phi_i$ are non-convex, we also prove linear convergence rate as long as the sum-of-non-convex objectives $f$ is convex.

We make the following assumptions for the primal problem (\ref{primal}) for further analysis. All of them are common assumptions in the theoretical analysis for the asynchronous stochastic methods.
\begin{assumption} [Lipschitz Constant]
	We assume $\nabla \phi_i$ is Lipschitz continuous, and there is Lipschitz constant $L$ such that $\forall x, y \in \mathbb{R}^d$:
	\begin{eqnarray}
		\|\nabla \phi_i(x) - \nabla \phi_i(y)\|_2 \leq L \|x-y \|_2
	\end{eqnarray}
	We can also know that $P$ is $(L+\lambda)$-smooth:
	\begin{eqnarray}
		\|\nabla P(x) - \nabla P(y)\|_2 \leq (L+\lambda) \|x-y \|_2
	\end{eqnarray}

\end{assumption}
\begin{assumption} [Maximum Time Delay]
	We assume that the maximum time delay of the global variable in each worker is upper bounded by $\tau$, such that:
	\begin{eqnarray}
		d(t) \geq t - \tau
	\end{eqnarray}
	$\tau$ is relevant to the number of workers $K$ in the system. We can also control $\tau$ through inner iteration $T$ in our algorithm.
	\label{time_delay}
\end{assumption}
\subsection{Convex Case}
In this section, we assume that the losses $\phi_i$ are convex, and prove that our method admits linear convergence.
\begin{assumption} [Convexity]
	We assume  losses $\phi_i$ are convex, such that $\forall x , y \in \mathbb{R}^d$:
	\begin{eqnarray}
		\phi_i(x) \geq \phi_i(y) + \nabla \phi_i(y)^T(x-y)\,.
	\end{eqnarray}
\end{assumption}
In our algorithm, dual variables $\alpha_{[1]},..., \alpha_{[K]}$ are stored in local workers. For worker $k$, there is no update of $\alpha_{[k]}$ from $d(t)$ to $t$. Therefore, it is always true that $\alpha_{[k]}^{s,t}= \alpha_{[k]}^{s,d(t)}$.
For brevity, we write $v^{s,t}$, $w^{s,t}$ and $\alpha^{s,t}$ as $v^t$, $w^t$ and $\alpha^t$. According to our algorithm, we know that:
\begin{eqnarray}
	v^{t} =  \sum\limits_{i\in I_t} \left(\nabla \phi_i(w^{d(t)}) + \alpha_i^{d(t)} \right)  =
	\sum\limits_{i\in I_t} v_i^t
\end{eqnarray}
where $|I_t|=H$ and $\mathbb{E}[v_i^{t}] = \nabla P(w^{d(t)})$. In our analysis, we also assume that there are no duplicate samples in $I_t$.
\iffalse
Thus, in each iteration, the update of  local variable $\alpha_i$ and global variable $w$ can be written as follows:
\begin{eqnarray}
	\alpha_i^{t+1} &=& \alpha_i^{t} -  \eta \lambda n (\nabla \phi_i(w^{d(t)}) + \alpha_i^{t-1}), i \in I_t  \\
	w^{t+1} &=& w^{t} -  \eta v^t
\end{eqnarray}
\fi
To analyze the convergence rate of our method, we need to prove  the following Lemma \ref{lem2} at first.	
\begin{lemma}
	\label{lem2}
	Let $w^*$ be the global solution of $P(w)$, and  $\alpha_i^*=-\nabla \phi_i(w^*)$.  Following the proof in \cite{shalev2015sdca}, we define $A_t$ and $B_t$ as follows:
	\begin{eqnarray}
		A_t &= &\mathbb{E} \| \alpha_i^{t} - \alpha_i^* \|^2\\
		B_t& = & \mathbb{E}\|w^{t} - w^* \|^2
	\end{eqnarray}
	According to  our algorithm, we can prove that $A_{t+1}$ and $B_{t+1}$ are upper bounded:
	\begin{eqnarray}
		\label{ineq_a}
			\mathbb{E}[A_{t+1} - A_{t} ] 
		&\leq&  - \eta \lambda H  \mathbb{E}\|\alpha_i^{t} - \alpha_i^*\|^2 - 2\eta  HL \lambda^2 \mathbb{E}\|w^t - w^*\|^2 + 4\eta \lambda H  L \left( P(x^t) - P(w^*)\right)\nonumber \\
		&& - \eta \lambda(1-\eta \lambda n) \mathbb{E}\|v^t\|^2 + 2 \lambda \tau H L^2 \eta^3 \sum\limits_{j=d(t)}^{t-1}\mathbb{E}\|v^j\|^2
	\end{eqnarray}
	\begin{eqnarray}
		\label{ineq_b}
			\mathbb{E} [B_{t+1} - B_{t}] 
		& \leq & -2\eta \left( P(w^{d(t)}) - P(w^*)\right)   + \eta^2 \mathbb{E} \|v^t\|^2
		  -2\eta \left< w^t - w^{d(t)}, \nabla  P(x^{d(t)})\right>
	\end{eqnarray}
\end{lemma}
\begin{theorem}
	\label{them_convex}
	Suppose losses $\phi_i$ are convex and $\nabla \phi_i$ are Lipschitz continuous.  Let $w^*$ be the optimal solution to $P(w)$, and  $\alpha_i^*=-\nabla \phi_i(w^*)$. Define $C_t = \frac{1}{2\lambda L} A_t+ B_t$. We can prove that as long as:
	\begin{eqnarray}
		\eta \leq \frac{1}{4HL\tau^2 + \lambda n + 2L }
	\end{eqnarray}
	the following inequality holds:
	\begin{eqnarray}
		\label{theom_iq1}
		\mathbb{E}[C_T] \leq (1-\eta \lambda H) \mathbb{E}[C_0]
	\end{eqnarray}
\end{theorem}
\begin{proof}\footnote{We provide the proof sketch here, please check the supplementary material for details.}
	Substituting $A_{t+1}$ and $B_{t+1}$ according to Lemma \ref{lem2}, the following inequality holds that:
	\begin{eqnarray}
		\mathbb{E}[C_{t+1}] &=&  \frac{1}{2\lambda L} A_{t+1} + B_{t+1} \nonumber \\
		&\leq & (1-\eta \lambda H) \mathbb{E}[C_t]  +   2\tau HL\eta^3   \sum\limits_{j=d(t)}^{t-1} \mathbb{E}\|v^j\|^2 +\left(  \frac{\eta^2 \lambda n}{2L}  + \eta^2 - \frac{\eta}{2L}   \right) \mathbb{E}\|v^t\|^2
	\end{eqnarray}
	Adding the above inequality from $t=0$ to $t=T-1$, we have:
	\begin{eqnarray}
			\sum\limits_{t=0}^{T-1} \mathbb{E} [C_{t+1}] &	\leq& \sum\limits_{t=0}^{T-1} (1-\eta \lambda H)\mathbb{E}[C_t]+ \left(2H\tau^2 \eta^2 + \frac{\eta^2 \lambda n}{2L}  + \eta^2 - \frac{\eta}{2L}   \right)  	\sum\limits_{t=0}^{T-1} \mathbb{E}\|v^t\|^2	
	\end{eqnarray}
	where the inequality follows from Assumption \ref{time_delay}  and $\eta L \leq 1$.
	If $2H\eta^2\tau^2 + \frac{\eta^2 \lambda n}{2L}  + \eta^2  - \frac{\eta}{2L}   \leq 0$, such that:
	\begin{eqnarray}
		\eta \leq \frac{1}{4HL\tau^2 + \lambda n + 2L },
	\end{eqnarray}
	we have the following inequality:
	\begin{eqnarray}
		\sum\limits_{t=0}^{T-1}	\mathbb{E} [C_{t+1}] &\leq& \sum\limits_{t=0}^{T-1} (1-\eta \lambda H)  \mathbb{E}[C_t] \nonumber \\
		&\leq & \sum\limits_{t=1}^{T-1}  \mathbb{E}[C_t]  + (1-\eta \lambda H)  C_0
	\end{eqnarray}
	Because $C_t \geq 0$, then we complete the proof that
	$\mathbb{E}[C_T] \leq (1- \eta \lambda H) \mathbb{E} [C_0].$   \QEDB \\
\end{proof}

Because $\nabla P(w)$ is  Lipschitz continuous, we know that:
\begin{eqnarray}
	P(w^t) - P(w^*) \leq \frac{L+\lambda}{2} \|w^t - w^*\|^2 \leq \frac{L+\lambda}{2} C_t
\end{eqnarray}
\begin{theorem}
	\label{them_convex_2}
	We consider the outer iteration $s$, and write $C^{t}$ as $C^{s,t}$. According to Algorithm \ref{alg2}, we know $C^{s+1, 0} = C^{s, T}$. Following Theorem \ref{them_convex} and applying (\ref{theom_iq1})  for $S$ iterations, it is satisfied that:
	\begin{eqnarray}
		\mathbb{E}[C_{S, 0}] \leq (1- \eta \lambda H)^S \mathbb{E}[C_{0, 0}]
	\end{eqnarray}
	In particular, to achieve $\mathbb{E} [P(w^{S, 0})  - P(w^*)] \leq \varepsilon$, it suffices to set
	$\eta = \frac{1}{4HL\tau^2 + \lambda n + 2L }$ and
	\begin{eqnarray}
		\label{con_s}
		S \geq  O\left( \left( \frac{L}{\lambda} \left(\tau^2 + \frac{1}{H} \right)  + \frac{n}{H}  \right)  \log \left(\frac{1}{\varepsilon}\right) \right)
	\end{eqnarray}
\end{theorem}
From Theorem \ref{them_convex} and  \ref{them_convex_2}, we know that our Dis-dfSDCA admits linear convergence if losses $\phi_i$ are convex. According to Theorem \ref{them_convex_2}, we observe that $\tau$ affects the speed of our convergence, if $\tau \rightarrow \infty$, it may lead our algorithm to diverge. Therefore, it is important to keep $\tau$ within a reasonable bound. In our algorithm, $\tau$ is relevant to the number of workers and less than $T$. When we let $H=1$ and $\tau=0$, $S$ is relevant to $O(\frac{L}{\lambda}+n)$. It is  compatible with the convergence analysis of sequential dfSDCA in \cite{shalev2015sdca}.

\iffalse
From \cite{shalev2015sdca}, dis-dfSDCA is also an instance of variance reduced SGD method. Right now, we can derive  the upper bound of the variance of  $v_i^t$.

\begin{corollary}
	Let $w^*$ be the optimal solution to problem (\ref{primal}) and $\alpha_i^*  = -\nabla \phi_i(x_i^Tw^*)$, as $t$ goes to infinity, $\mathbb{E}\|v_i^t\|^2$ goes to zero:
	\begin{eqnarray}
		&&\mathbb{E}\| v_i^t - \nabla P(w^{d(t)}) \|^2 = \mathbb{E} \|v_i^t\|^2 - \|\nabla P(w^{d(t)})\|^2 \nonumber \\
		&\leq &  \mathbb{E} \|\alpha_i^{t}x_i - \alpha_i^* x_i +  \alpha_i^* x_i +  \nabla \phi_i(x_i^Tw^{d(t)})x_i \|^2 \nonumber \\
		&\leq & 2\mathbb{E} \|\alpha_i^{t-1} x_i - \alpha_i^* x_i \|^2 + 2\mathbb{E} \|\alpha_i^{*}x_i+ \nabla \phi_i(x_i^Tw^{t-\tau})x_i \|^2 \nonumber \\
		&\leq & 2\mathbb{E} \|\alpha_i^{t-1} x_i - \alpha_i^* x_i \|^2 + 2L^2\mathbb{E} \|w^{t-\tau} - w^* \|^2 \nonumber \\
	\end{eqnarray}
\end{corollary}
where the second inequality follows from the Lipschitz continuity of $\nabla \phi_i$. From Theorem \ref{them_convex}, both of $\mathbb{E} \|\alpha_i^{t-1} x_i - \alpha_i^* x_i \|^2$ and $\mathbb{E} \|w^{t-\tau} - w^* \|^2$ go to zero linearly. Therefore, the variance of $v_i^t$ is close to zero finally.
\fi

\subsection{Non-convex Case}
In this section,  we assume that the losses $\phi_i$ are non-convex, while the sum-of-non-convex objectives $f$ is convex. We also prove that Dis-dfSDCA admits linear convergence rate for this case. Firstly, we get the following Lemma \ref{non_lem2}.
\begin{lemma}
	\label{non_lem2}
	Let $w^*$ be optimal solution to $P(w)$, and  $\alpha_i^*=-\nabla \phi_i(w^*)$.
	Following the definition of  $A_t$ and $B_t$ in Lemma \ref{lem2}, we prove that $A_{t+1}$ and $B_{t+1}$ are upper bounded:
	\begin{eqnarray}
		\label{non_ineq_a}
			\mathbb{E}[A_{t+1} - A_{t} ] 
		&\leq&  - \eta \lambda H  \mathbb{E}\|\alpha_i^{t}- \alpha_i^*\|^2  + 2\eta \lambda H L^2\mathbb{E}\|w^t - w^*\|^2  \nonumber \\
		&&- \eta \lambda(1-\eta \lambda n) \mathbb{E}\|v^t\|^2  + 2\lambda \tau HL^2 \eta^3 \sum\limits_{j=d(t)}^{t-1}\mathbb{E}\|v^j\|^2 
	\end{eqnarray}
	\begin{eqnarray}
		\label{non_ineq_b}
			\mathbb{E} [B_{t+1} - B_{t}] 
		& \leq & - \frac{3\eta \lambda H}{4} \mathbb{E}\|w^t - w^* \|^2  + \eta^2 \mathbb{E} \|v^t\|^2
		 \frac{2H \tau H^2 (L+\lambda)^2 \eta^3}{\lambda} \sum\limits_{j=d(t)}^{t-1}\mathbb{E}\|v^j\|^2
	\end{eqnarray}
\end{lemma}

\begin{theorem}
	\label{them_nonconvex}
	Suppose $f$ is convex  and $\nabla \phi_i$ is Lipschitz continuous. Let $w^*$ be the optimal solution to $P(w)$, and  $\alpha_i^*=-\nabla \phi_i(w^*)$. Define $C_t = \frac{1}{4 L^2} A_t+ B_t$. We can prove that as long as:
	\begin{eqnarray}
		\eta \leq \frac{\lambda^2}{2HL\tau^2\lambda^2 + 8 HL\tau^2(L+\lambda)^2  + 4\lambda L^2 + n \lambda^3 }
	\end{eqnarray}
	the following inequality holds:
	\begin{eqnarray}
		\label{non_theom_iq1}
		\mathbb{E}[C_T] \leq (1-\eta \lambda H) \mathbb{E}[C_0]
	\end{eqnarray}
\end{theorem}

\begin{proof}
	Substituting $A_{t+1}$ and $B_{t+1}$ according to  Lemma \ref{non_lem2}, the following inequality holds that:
	\begin{eqnarray}
			\mathbb{E}[C_{t+1}] &=&  \frac{1}{4 L^2} A_{t+1} + B_{t+1} \nonumber \\
		&\leq & (1-\eta \lambda H) \mathbb{E}[C_t] +  \left(\frac{\lambda H\tau \eta^3}{2} + \frac{2H\tau (L+\lambda)^2 \eta^3}{\lambda} \right)   \sum\limits_{j=d(t)}^{t-1} \mathbb{E}\|v^j\|^2 \nonumber \\
		&&+\left(  \eta^2 + \frac{n \eta^2 \lambda^2}{4L^2} - \frac{\eta \lambda }{4L^2}   \right) \mathbb{E}\|v^t\|^2
	\end{eqnarray}
	Adding the above inequality from $t=0$ to $t=T-1$, we have:
	\begin{eqnarray}
		\sum\limits_{t=0}^{T-1} \mathbb{E} [C_{t+1}] &	\leq& \sum\limits_{t=0}^{T-1} (1-\eta \lambda H)\mathbb{E}[C_t] \nonumber\\
		&&+ \biggl(\eta^2 + \frac{n \eta^2 \lambda^2}{4L^2}  
		+	\frac{\lambda H\tau^2 \eta^2}{2L} + \frac{2H\tau^2 (L+\lambda)^2 \eta^2}{\lambda L}- \frac{\eta \lambda }{4L^2}     \biggr)  	\sum\limits_{t=0}^{T-1} \mathbb{E}\|v^t\|^2	
	\end{eqnarray}
	where the inequality follows from Assumption \ref{time_delay}  and $\eta L \leq 1$.
	If $\eta^2 + \frac{n \eta^2 \lambda^2}{4L^2}  +
	\frac{\lambda H\tau^2 \eta^2}{2L} + \frac{2H\tau^2 (L+\lambda)^2 \eta^2}{\lambda L}- \frac{\eta \lambda }{4L^2}     \leq 0$, such that:
	\begin{eqnarray}
		\eta \leq \frac{\lambda^2}{2HL\tau^2\lambda^2 + 8 HL\tau^2 (L+\lambda)^2 + 4\lambda L^2 + n \lambda^3 }
	\end{eqnarray}
	we have the following inequality:
	\begin{eqnarray}
		\sum\limits_{t=0}^{T-1}	\mathbb{E} [C_{t+1}] &\leq& \sum\limits_{t=0}^{T-1} (1-\eta \lambda H)  \mathbb{E}[C_t] \nonumber \\
		&\leq & \sum\limits_{t=1}^{T-1}  \mathbb{E}[C_t]  + (1-\eta \lambda H)  C_0
	\end{eqnarray}
	Because $C_t \geq 0$, then we complete the proof that
	$\mathbb{E}[C_T] \leq (1- \eta \lambda H) \mathbb{E} [C_0]$.\QEDB \\
\end{proof}

\begin{theorem}
	\label{them_nonconvex_2}
	We consider the outer iteration $s$, and write $C^{t}$ as $C^{s,t}$. According to Algorithm \ref{alg2}, we know $C^{s+1, 0} = C^{s, T}$. Following Theorem \ref{them_nonconvex} and applying (\ref{non_theom_iq1})  for $S$ iterations, it is satisfied that:
	\begin{eqnarray}
		\mathbb{E}[C_{S, 0}] \leq (1- \eta \lambda H)^S \mathbb{E}[C_{0, 0}]
	\end{eqnarray}
	In particular, to achieve $\mathbb{E} [P(w^{S, 0})  - P(w^*)] \leq \varepsilon$, it suffices to set
	$\eta = \frac{\lambda^2}{2HL\tau^2\lambda^2 + 8 HL\tau^2 (L+\lambda)^2 + 4\lambda L^2 + n \lambda^3 }$ and
	\small
	\begin{eqnarray}
		\label{non_s}
		S \geq  O\left( \left(   \frac{\left( \tau^2+ 1/H \right) L^2}{\lambda^2}  + \frac{\tau^2L^3}{\lambda^3} + \frac{n}{H}  \right)  \log \left(\frac{1}{\varepsilon}\right) \right)
	\end{eqnarray}
\end{theorem}
From Theorem \ref{them_nonconvex} and  \ref{them_nonconvex_2}, we know that our Dis-dfSDCA admits linear convergence even if  losses $\phi_i$ are non-convex, as long as the sum-of-non-convex objectives is convex. Comparing Theorem \ref{them_convex_2} with \ref{them_nonconvex_2}, we can observe that our method needs more iterations to converge to the similar accuracy when $\phi_i$ are non-convex. It is reasonable because non-convex problem is known to be harder to be optimized than convex problem. When we let $H=1$ and $\tau=0$,  $S$ is relevant to $O(\frac{L^2}{\lambda^2}+n)$. It is also compatible with the convergence analysis of sequential dfSDCA in \cite{shalev2015sdca}.

\begin{figure}[h]
	\centering
	\begin{subfigure}[b]{0.31\textwidth}
		\centering
		\includegraphics[width=2.2in]{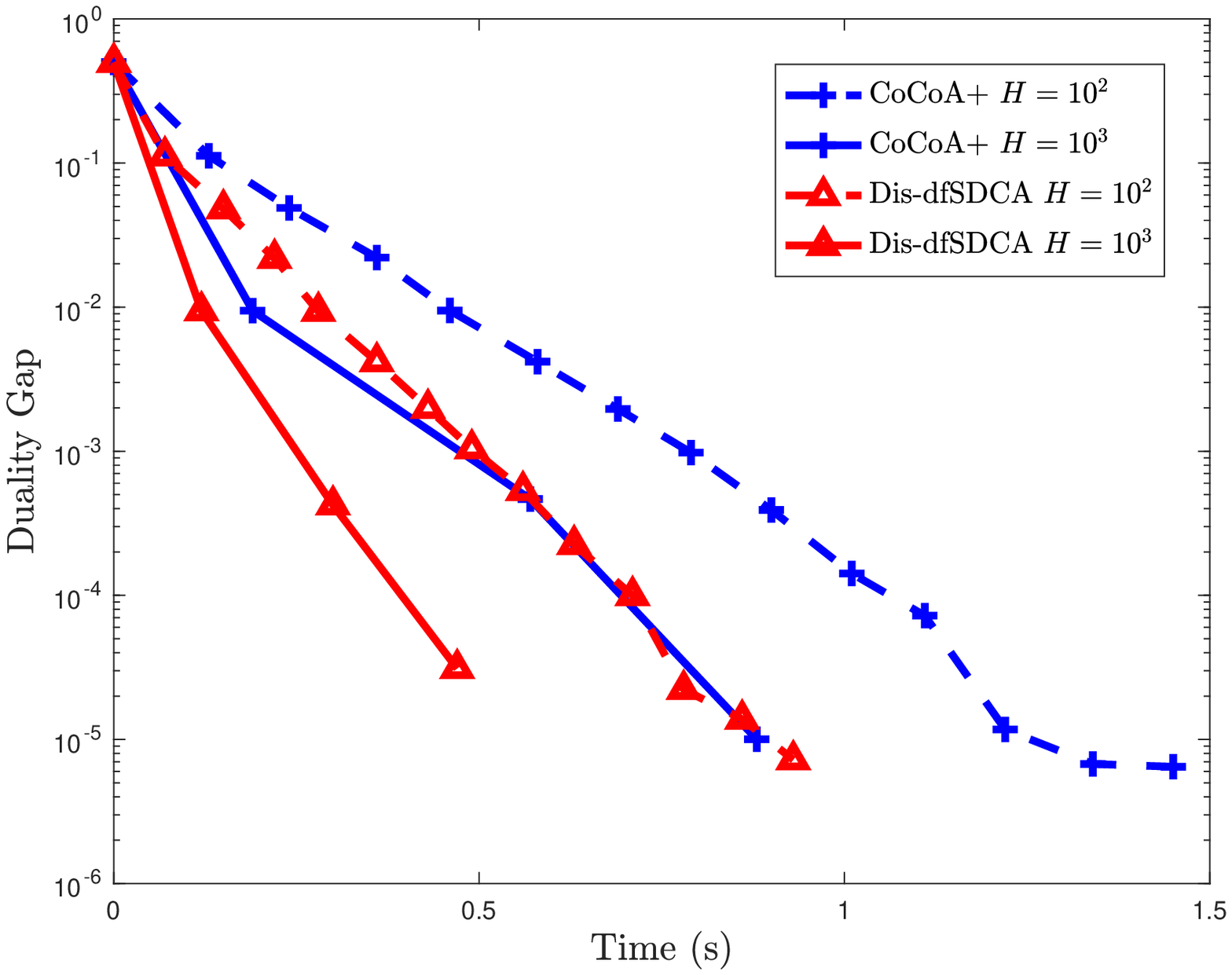}
		%\caption{IJCNN1}
		%\label{ijcnn1_gap_time}
	\end{subfigure}
	\begin{subfigure}[b]{0.31\textwidth}
		\centering
		\includegraphics[width=2.2in]{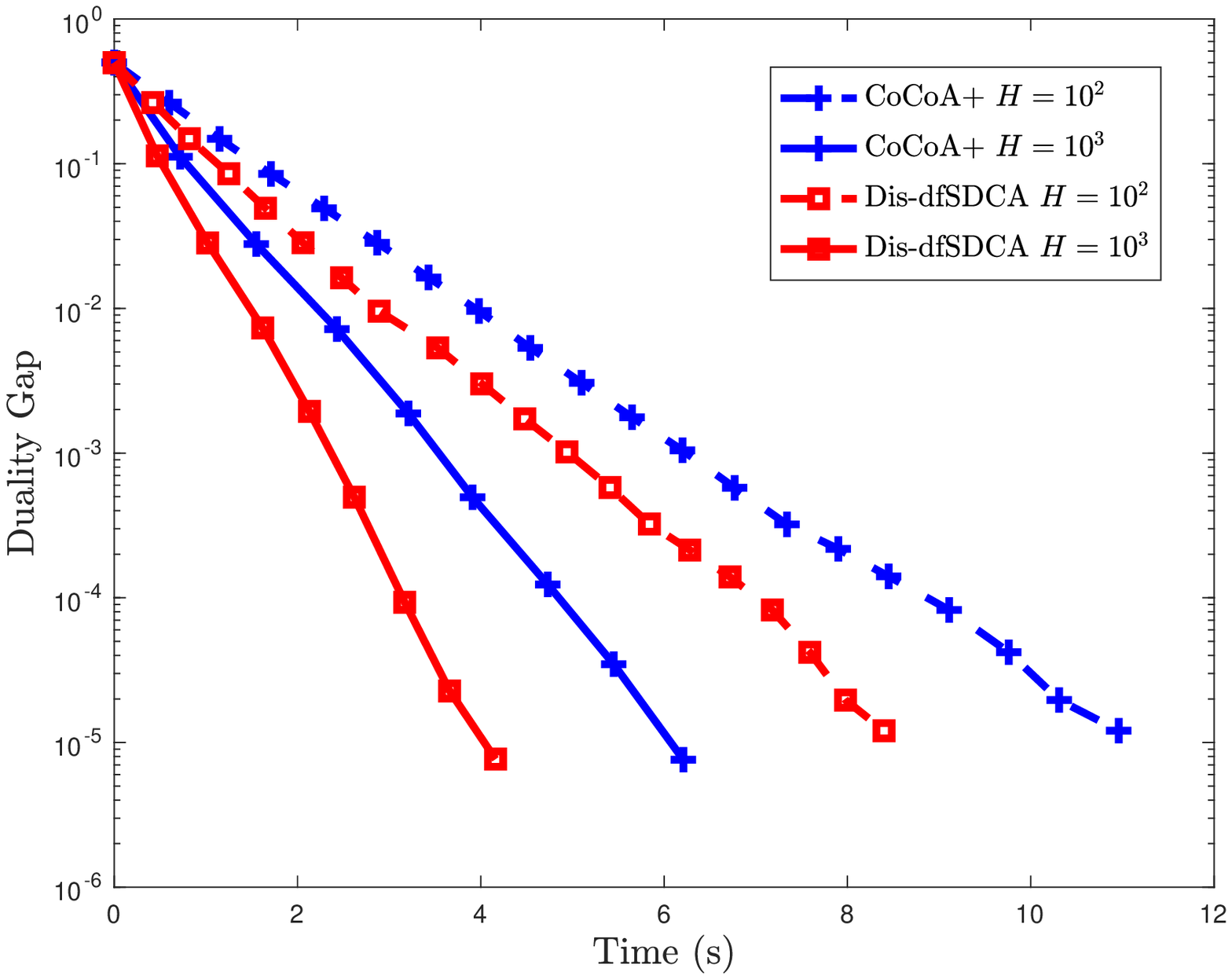}
		%\caption{COVTYPE}
	\end{subfigure}
	\begin{subfigure}[b]{0.31\textwidth}
		\centering
		\includegraphics[width=2.2in]{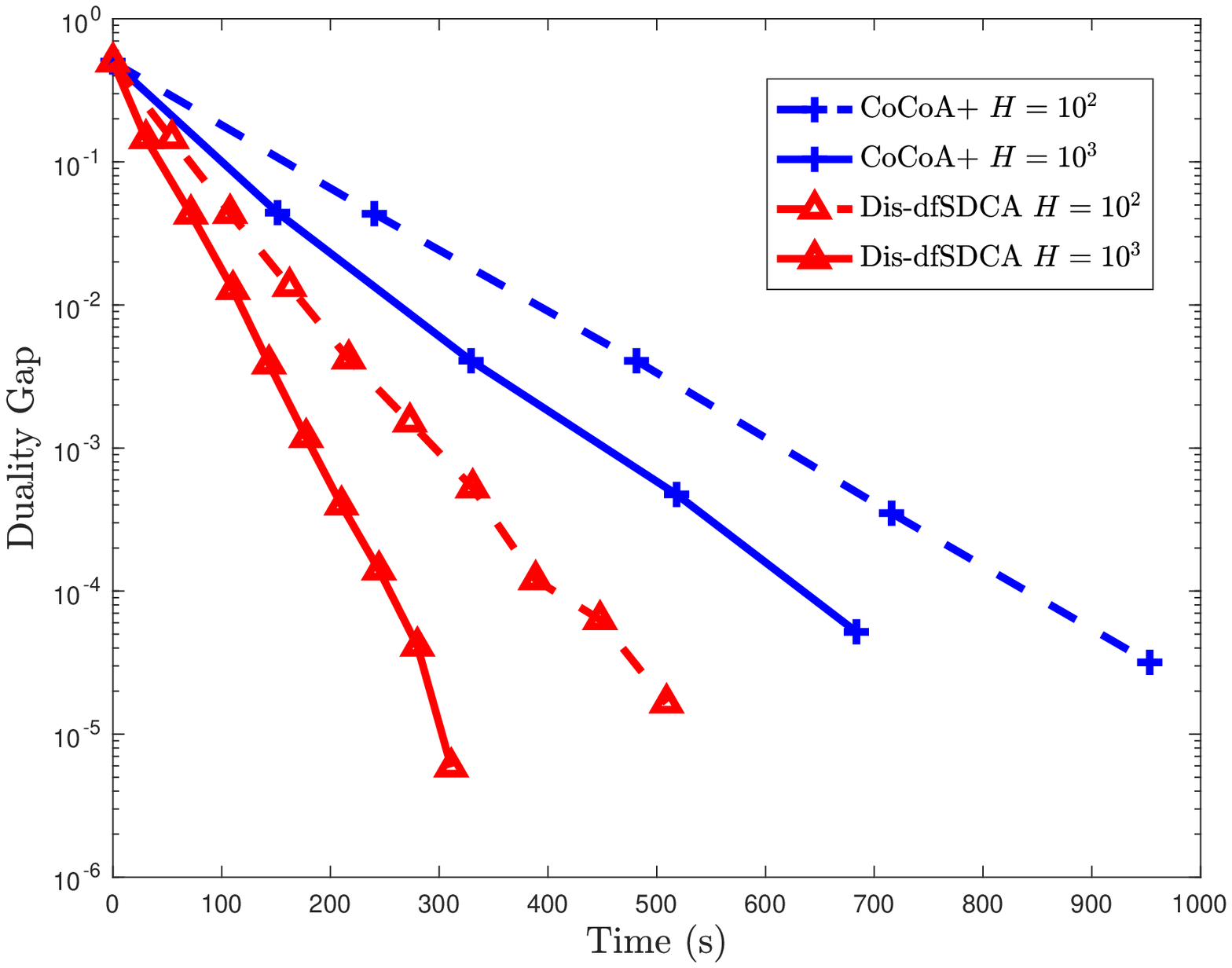}
		%\caption{RCV1}
	\end{subfigure}
	\begin{subfigure}[b]{0.31\textwidth}
		\centering
		\includegraphics[width=2.2in]{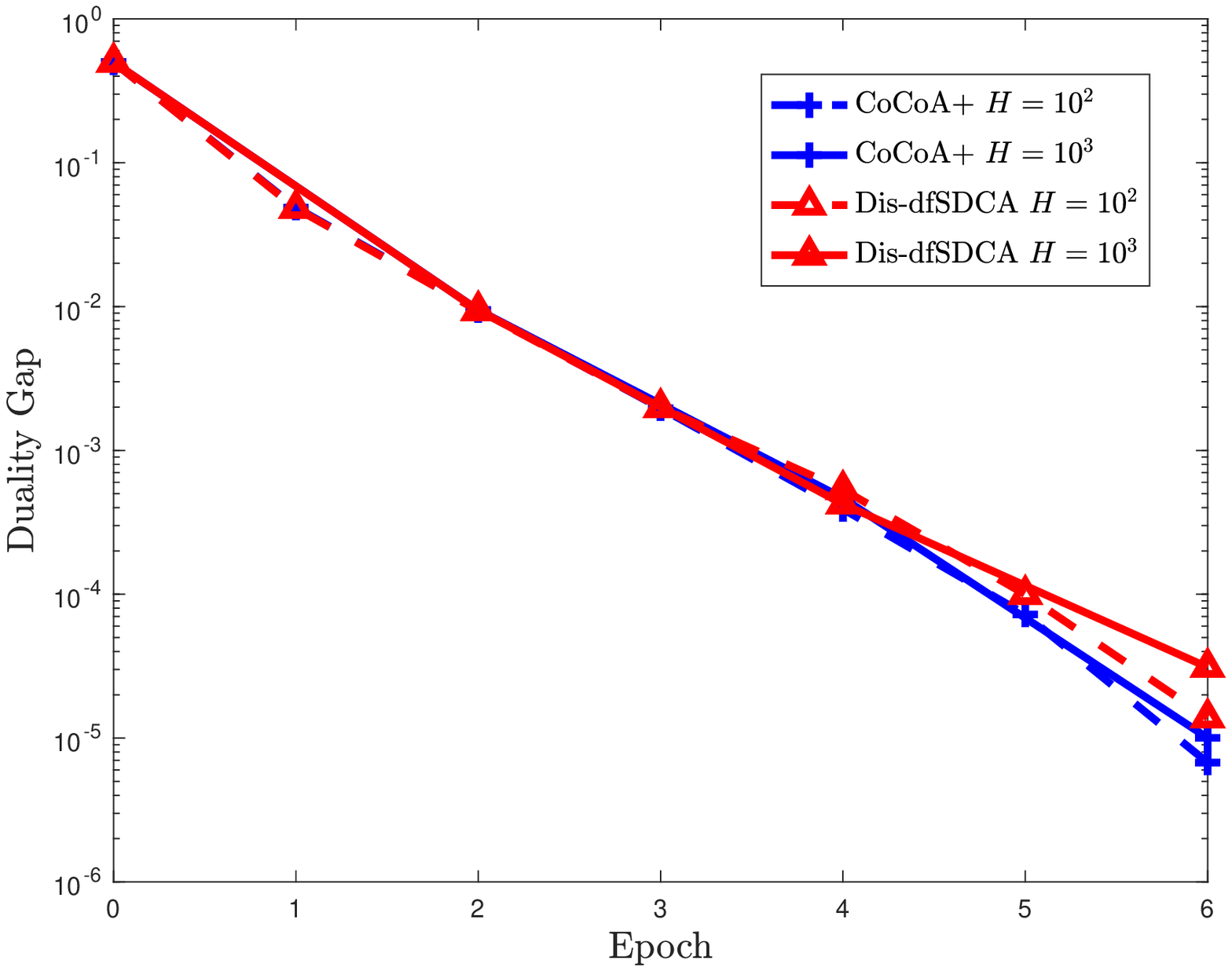}
		\caption{IJCNN1}
		\label{ijcnn1_gap_epoch}
	\end{subfigure}
	\begin{subfigure}[b]{0.31\textwidth}
		\centering
		\includegraphics[width=2.2in]{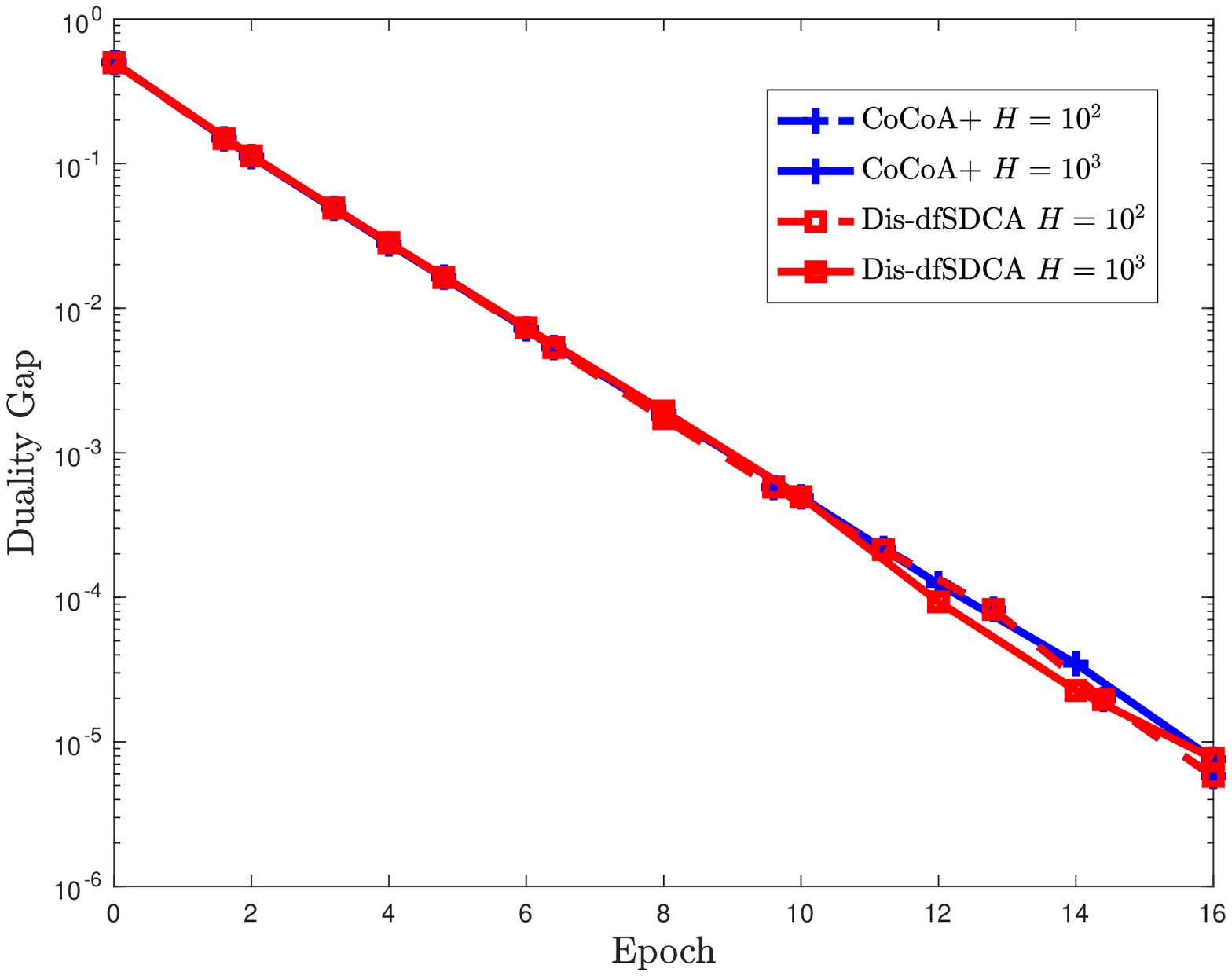}
		\caption{COVTYPE}
	\end{subfigure}
	\begin{subfigure}[b]{0.31\textwidth}
		\centering
		\includegraphics[width=2.2in]{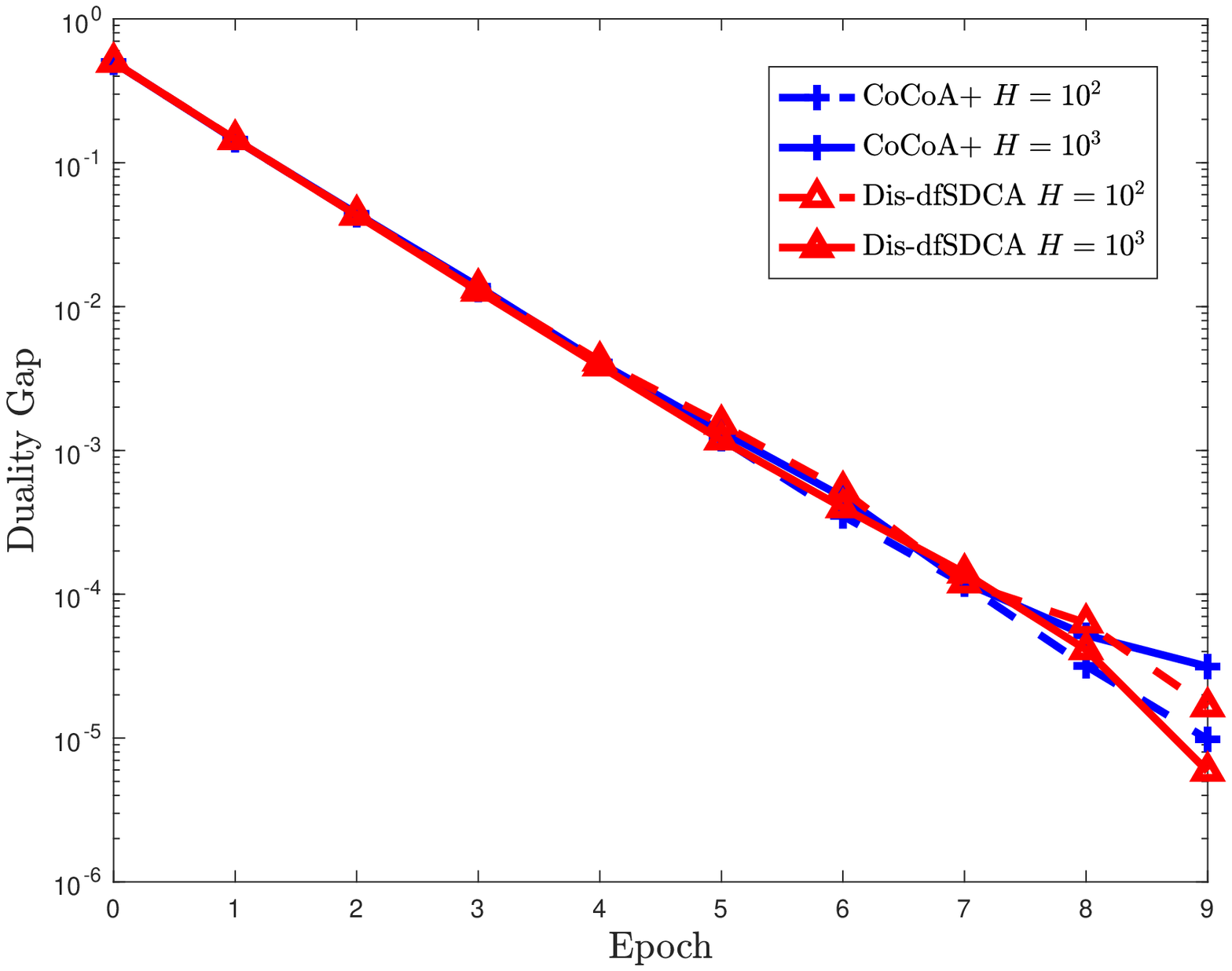}
		\caption{RCV1}
	\end{subfigure}
	\iffalse
	\begin{subfigure}[b]{0.33\textwidth}
		\centering
		\includegraphics[width=2.45in]{ijcnn_gradient_time}
		\caption{IJCNN1}
	\end{subfigure}
	\begin{subfigure}[b]{0.33\textwidth}
		\centering
		\includegraphics[width=2.45in]{covtype_gradient_time}
		\caption{COVTYPE}
	\end{subfigure}
	\begin{subfigure}[b]{0.33\textwidth}
		\centering
		\includegraphics[width=2.45in]{rcv1_gradient}
		\caption{RCV1}
	\end{subfigure}
	\fi
	%\caption{Convergence of full gradient in terms of time. }
	\caption{Figures (a) - (c) present the convergence of  duality gap of compared methods in terms of time. Figures (d) - (f) present the convergence of  duality gap of compared methods in terms of epoch number. %Figures (g) - (i) present the convergence of full gradient Euclidean norm in terms of time.
		We train IJCNN1 dataset with $4$ workers, COVTYPE  dataset with $8$ workers and RCV1 dataset with $16$ workers.}
	\label{convergence}
\end{figure}

\section{Experiments }
In this section, we conduct two simulated experiments on the distributed system with straggler problem. There are mainly three goals, firstly, we want to verify that our Dis-dfSDCA has linear convergence rate for the convex and smooth problem; secondly, we would like to make sure that our method has better speedup property than other primal-dual methods; thirdly, we would like to show that our method is also fit for non-convex losses.

Our algorithm is implemented using C++, and the point-to-point communication between worker and server is handled by openMPI \cite{gabriel04:_open_mpi}. We use Armadillo library \cite{sanderson2016armadillo} for efficient matrix computation.  Experiments are performed on Amazon Web Services, and each node is a t2.medium instance which has two virtual CPUs.  In our distributed system, we simulate the straggler problem by forcing one selected worker node to the delaying state for $m$ times as long as the normal computing time of other normal workers with probability $p$. In our experiments, we set $p=0.2$ and $m$ is selected from $[0,10]$ randomly. In practice, all nodes have a tiny possibility of being delayed. The setting in our experiments is to verify that our algorithm is robust to straggler problem, even in the extreme situation.

\subsection{Convex Case}
In our experiment, we optimize quadratic loss with $\ell_2 $ regularization term to solve binary classification problem:
\begin{eqnarray}
	\min_{w \in \mathbb{R}^d} \frac{1}{n} \sum\limits_{i=1}^n \frac{1}{2}(x_i^T w - y_i)^2 +  \frac{\lambda}{2} \| w\|^2
	\label{problem}
\end{eqnarray}
where $\lambda=0.1$. Datasets in our experiments are from LIBSVM \cite{CC01a}. %\footnote{https://www.csie.ntu.edu.tw/cjlin/libsvmtools/datasets/binary.html}.
Table \ref{table_data} shows brief details of each dataset. In this problem, because $\nabla \phi_i(w)$ can be written as $\nabla \phi_i(x_i^Tw)$, we just need to store $\hat \alpha \in \mathbb{R}^n$, and recover $\alpha \in \mathbb{R}^{d \times n} $ through $a_i = x_i \hat \alpha_i$. Therefore the space complexity is $O(n)$.

We compare our method with CoCoA+ \cite{ma2015adding}, which is the state-of-the-art distributed primal-dual optimization framework. We reimplement CoCoA+ framework using C++, and use SDCA as the local solver. Learning rate $\eta$ in our method is selected from $\eta = \{1,0.1,0.001,0.0001\}$.
\begin{table}[h]
	\center
	\begin{tabular}{c|c|c|c}
		\hline
		\hline
		Dataset & $\#$ of samples & Dimension & Sparsity \\
		\hline
		IJCNN1 & 49,990 & 22 &  41 \% \\
		\hline
		COVTYPE & 581,012 & 54 & 22 \% \\
		\hline
		RCV1 & 677,399 & 47,236 & 0.16\% \\
		\hline
		\hline
	\end{tabular}
	\caption{Experimental datasets from LIBSVM. }
	\label{table_data}
\end{table}

\subsubsection{Convergence of Duality Gap}
We compare the duality gap convergence of compared methods in terms of time and epoch number respectively, where duality gap is well defined in \cite{shalev2013stochastic}. Experimental results are presented in Figure \ref{convergence}.
We distribute IJCNN1 dataset over $4$ workers. Figures \ref{ijcnn1_gap_epoch}  in the first column show the duality gap convergence in terms of time and epoch on IJCNN1 dataset.  From the second figure, it is easy to know that Dis-dfSDCA and CoCoA+ have similar convergence rate. Since CoCoA+ has linear convergence if the problem is convex and smooth, it is verified that Dis-dfSDCA has linear convergence rate as well. In the experiment, we evaluate Dis-dfSDCA when we set different amount of local computations, $H=10^2$ and $H=10^3$. Results show that our method is faster than CoCoA+ method in both two cases. The reason is that CoCoA+ is affected by the straggler problem in the distributed system. We also optimize problem  (\ref{problem}) with COVTYPE dataset using $8$ workers, and RCV1 dataset using $16$ workers. We can draw the similar conclusion from the  results of  other two datasets.

\begin{figure}[t]
	\centering
	\begin{subfigure}[b]{0.45\textwidth}
		\centering
		\includegraphics[width=3in]{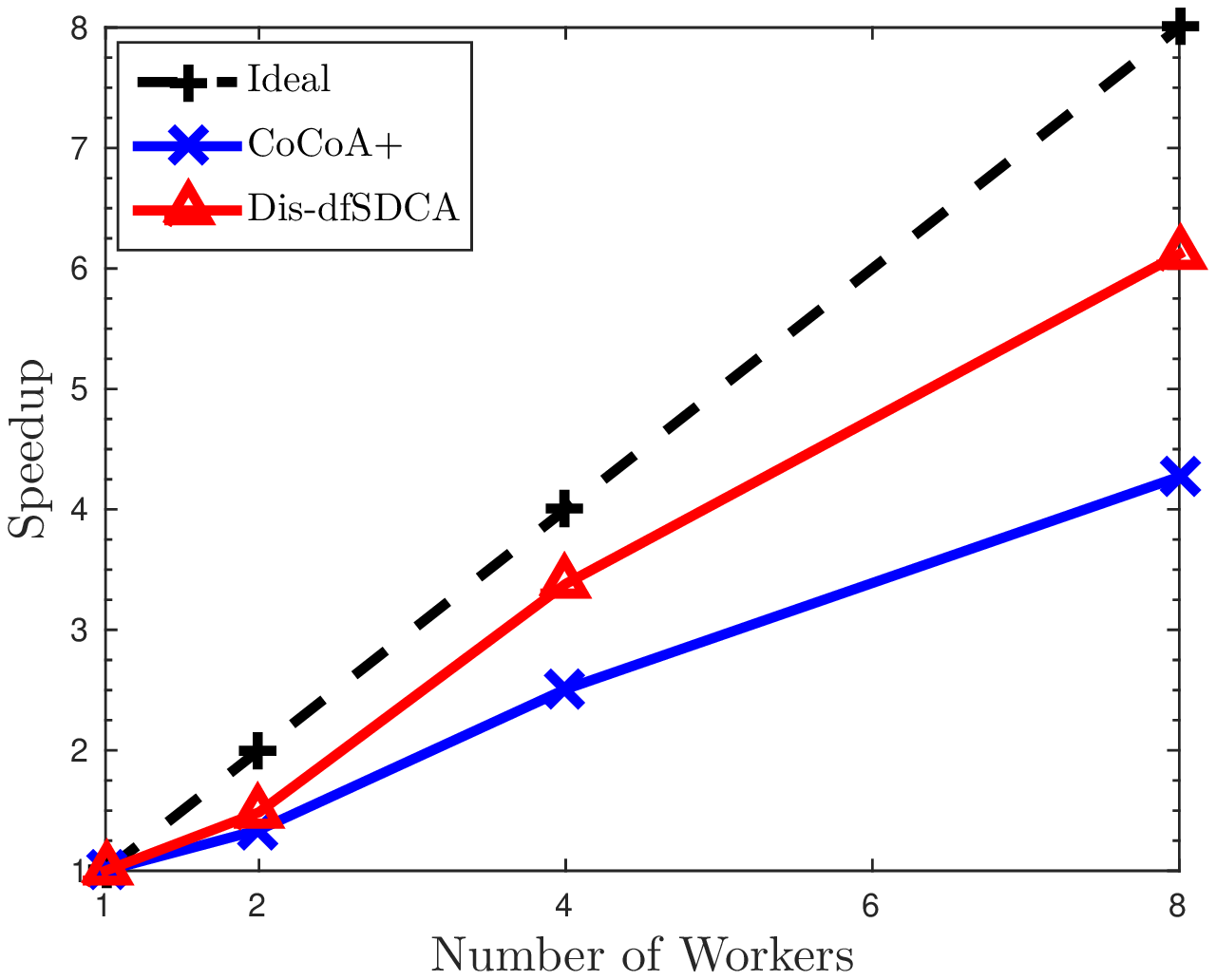}
	\end{subfigure}
	\begin{subfigure}[b]{0.45\textwidth}
		\centering
		\includegraphics[width=3in]{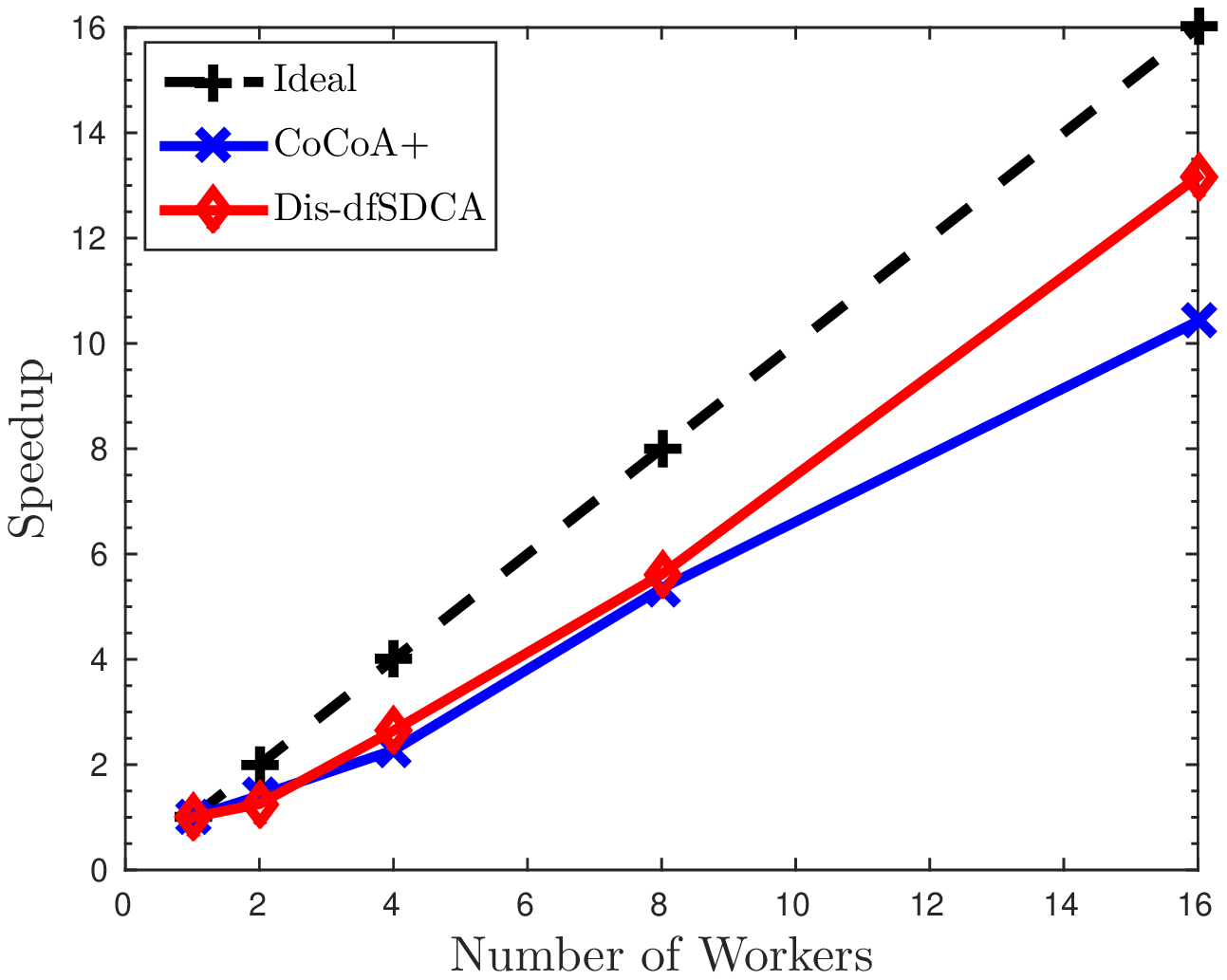}
	\end{subfigure}
	\begin{subfigure}[b]{0.8\textwidth}
		\includegraphics[width=5.5in]{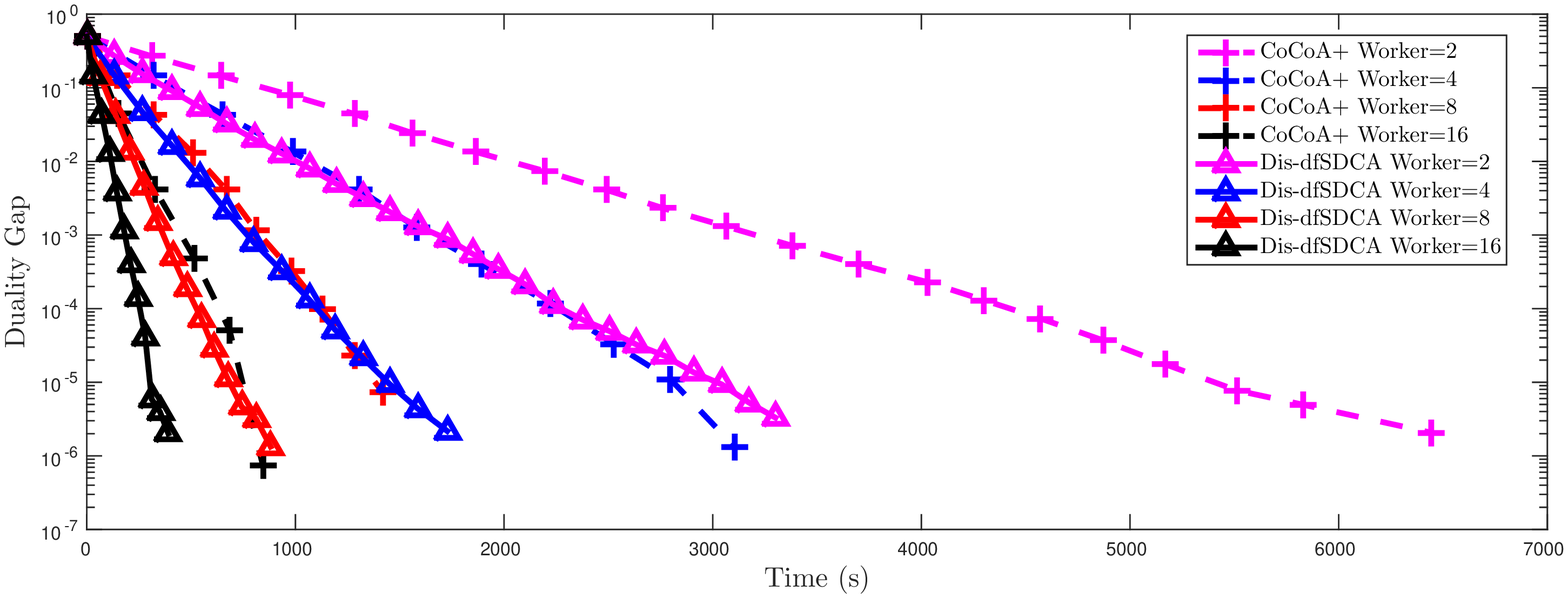}
	\end{subfigure}
	\caption{Time speedup in terms of the number of workers. Row $1$ left: IJCNN1; Row $1$ right: COVTYPE; Row $2$: RCV1. }
	\label{workers}
\end{figure}	

%\subsubsection{Convergence of Full Gradient}
%It is known that $\|\nabla P\|^2 \rightarrow 0$ when $w \rightarrow w^*$. Therefore, we also compare the convergence of $\|\nabla P\|^2$ for compared methods. The third row in Figure \ref{convergence} presents the convergence of $\|\nabla P\|^2$ in terms of  time. Firstly, it is easy to know that $\|\nabla P\|^2$  also converges to $0$ as duality gap value goes close to $0$; secondly, we observe that Dis-dfSDCA performs better than CoCoA+ from the view full gradient convergence.

\subsubsection{Speedup}
In this section, we evaluate the scaling up ability of compared methods. The first row of Figure \ref{workers} presents the speedup of compared methods on IJCNN1 and COVTYPE datasets.  Speedup is defined as follows:
\begin{eqnarray}
	\text{ Time speedup } = \frac{\text{Running time for serial computation}}{\text{Running time of using } K \text{ workers}}
\end{eqnarray}
Figure in the second row shows the convergence of duality gap on RCV1 on multiple machines. It is obvious that Dis-dfSDCA always converges faster than CoCoA+ when they have the same number of workers. Experimental results verify that Dis-dfSDCA has better speedup property than CoCoA+ when there is straggler problem.

\subsection{Non-convex Case}
In this experiment, we optimize the following convex objective, which is an essential step for principal component analysis in \cite{garber2015fast}:
\begin{eqnarray}
	\label{exp_pca}
	\min\limits_{w \in \mathbb{R}^d} \frac{1}{n}\sum\limits_{i=1}^n \frac{1}{2}w^T \left((\mu - \lambda)  - x_ix_i^T\right)w - b^Tw + \frac{\lambda}{2} \|w\|^2
\end{eqnarray}
We conduct the experiment on synthetic data and generate $n = 500,000$ random vectors $\{x_1, ..., x_{500,000} \}\in \mathbb{R}^{500}$ which are mean subtracted and normalized to have Euclidean  norm $1$.	$C= \frac{1}{n} \sum_{i=1}^n x_ix_i^T $ denotes covariance matrix, $b\in \mathbb{R}^d$ denotes a random vector and we let $\mu = 100$, $\lambda = 10^{-4}$ in the experiment.
\begin{figure}[t]
	\centering
	\includegraphics[width=4in]{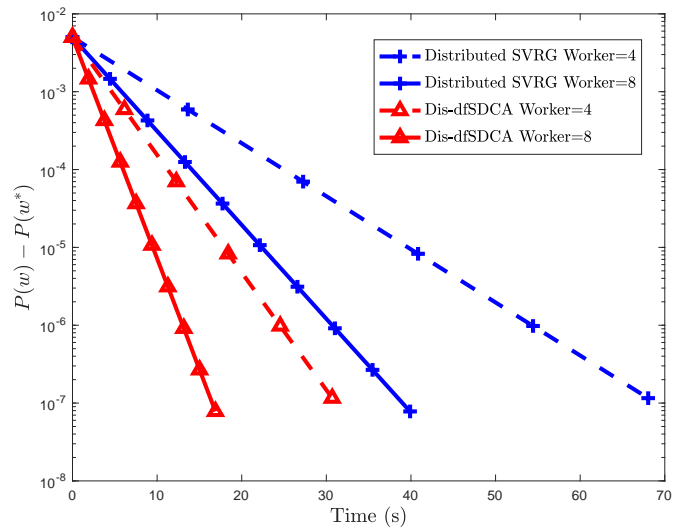}
	\caption{Suboptimum ($P(w) - P(w^*)$) convergence of compared methods in terms of time. $w^*$ denotes the optimal solution to problem (\ref{exp_pca}) , and it is obtained by running Dis-dfSDCA until convergence. }
	\label{pca}
\end{figure}
Because each $\phi_i$ is probably non-convex, CoCoA is not able to solve this problem. In this experiment, we compare with Distributed asynchronous SVRG \cite{huo2017asynchronous}.

In Figure \ref{pca}, it is obvious that Dis-dfSDCA runs faster than Distributed SVRG when there are $4$ workers. We can observe the similar phenomenon when there are $8$ workers. This observation is reasonable because Distributed SVRG needs to compute two gradients in each inner iteration and full gradient in each outer iteration. Dis-dfSDCA is faster because it only needs to compute one gradient in each iteration. However, Dis-dfSDCA needs $O(nd)$  space for storing $\alpha$ , because $\nabla \phi_i(w)$ cannot be written as $\nabla \phi_i(x_i^Tw)x_i$ in this problem.

\vspace*{-3pt}
\section{Conclusion}
In this paper, we proposed Distributed Asynchronous Dual Free Coordinate Ascent (Dis-dfSDCA) method for distributed machine learning. We addressed two challenging issues in previous primal-dual distributed optimization methods: firstly, Dis-dfSDCA does not rely on the dual formulation, and can be used to solve the non-convex problem; secondly, Dis-dfSDCA uses asynchronous communication and can be applied on the complicated distributed system where there is straggler problem. We also analyze the convergence rate of Dis-dfSDCA and prove linear convergence even if the loss functions are non-convex, as long as the sum of non-convex objectives is convex. We conduct experiments on the simulated distributed system with straggler problem, and all experimental results consistently verify our theoretical analysis.

\bibliographystyle{apalike}
\bibliography{ASDCA}

\appendix

\noindent \textbf{Proof to Lemma \ref{lem2}}
\begin{proof}
	In our proof, we suppose that there are no duplicate samples in $I_t$. According to our algorithm, we know that:  
	\begin{eqnarray}
	v^{t} =  \sum\limits_{i\in I_t} \left(\nabla \phi_i(w^{d(t)}) + \alpha_i^{d(t)} \right)  =
	\sum\limits_{i\in I_t} v_i^t
	\end{eqnarray}
	where $|I_t|=H$ and $\mathbb{E}[v_i^{t}] = \nabla P(w^{d(t)})$. 
	%Because we only consider the case of $H=1$, then $\alpha_i$ will be not updated in the process from $t-\tau$ to $t-1$, so $\alpha_i^{s,t-\tau} = \alpha_i^{s,t-1}$.
	Following the proof in \cite{shalev2015sdca}, we define $A_t$ and $B_t$ as follows:
	\begin{eqnarray}
	A_t &= &\mathbb{E} \| \alpha_i^{t}- \alpha_i^*\|^2\\
	B_t& = & \mathbb{E}\|w^{t} - w^* \|^2
	\end{eqnarray}
	Defining $\beta = \eta \lambda n$, so in iteration $t$, $\alpha_i^{t+1} = (1-\beta)\alpha_i^{t} + \beta (-\nabla \phi_i(w^{d(t)})) $, we have:
	\begin{eqnarray}
	&&\mathbb{E}[A_{t+1} - A_{t}] \nonumber \\
	&=&  \mathbb{E} \left[\frac{1}{n} \sum\limits_{i=1}^n  \| \alpha_i^{t+1} - \alpha_i^* \|^2  - \frac{1}{n}\sum\limits_{i=1}^n \| \alpha_i^{t} - \alpha_i^* \|^2\right]  \nonumber \\
	&=& \mathbb{E}\left[ \frac{1}{n} \sum\limits_{i \in I_t}\left( \| \alpha_i^{t+1} - \alpha_i^* \|^2  -  \| \alpha_i^{t} - \alpha_i^* \|^2 \right)  \right]\nonumber \\
	& =  &\mathbb{E}\left[ \frac{1}{n} \sum\limits_{i \in I_t} \| (1-\beta)(\alpha_i^{t} - \alpha_i^*)  + \beta (-\nabla \phi_i(w^{d(t)}) - \alpha_i^*) \|^2 - \frac{1}{n}\|\alpha_i^{t} - \alpha_i^* \|^2\right] \nonumber \\
	&=& \mathbb{E}\biggl[ \frac{1}{n}\sum\limits_{i \in I_t} \biggl( (1-\beta)\|\alpha_i^{t} - \alpha_i^*\|^2 + \beta\|\nabla \phi_i(w^{d(t)}) + \alpha_i^*\|^2 - \beta(1-\beta)\|\alpha_i^{t}+ \nabla \phi_i(w^{d(t)})\|^2 - \|\alpha_i^{t}- \alpha_i^*\|^2 \biggr) \biggr] \nonumber \\
	& =&  \eta \lambda H \biggl( - \mathbb{E}\|\alpha_i^{t} - \alpha_i^*\|^2 + \mathbb{E}\|\nabla \phi_i(w^{d(t)}) + \alpha_i^* \|^2 \biggr) - \eta \lambda(1-\beta) \sum\limits_{i \in I_t}\mathbb{E}\|v_i^t\|^2 \nonumber \\
	&\leq & \eta \lambda H \biggl( - \mathbb{E}\|\alpha_i^{t} - \alpha_i^*\|^2 + \mathbb{E}\|\nabla \phi_i(w^{d(t)}) + \alpha_i^* \|^2 \biggr) - \eta \lambda(1-\beta) \mathbb{E}\|v^t\|^2 
	\label{a_111}
	\end{eqnarray}
	where the last inequality follows from that 
	$\sum\limits_{i \in I_t}\mathbb{E}\|v_i^t\|^2 \geq \mathbb{E}\|\sum\limits_{i \in I_t} v_i^t\|^2=\mathbb{E}\| v^t\|^2 $.  
	Because $\alpha_i^* = -\nabla \phi_i(w^*)$, we have the following inequality:
	\begin{eqnarray}
	\label{ineq_1}
	&&\mathbb{E}\|\nabla \phi_i(w^{d(t)})+ \alpha^* \|^2 \nonumber \\
	&=&  \mathbb{E}[\|\ \nabla \phi_i(w^{d(t)}) - \nabla \phi_i(w^{*}) \|^2] \nonumber \\ 
	&\leq & 2\mathbb{E} \|\nabla \phi_i(w^{d(t)}) - \nabla \phi_i(w^{t}) \|^2  + 2 \mathbb{E}  \|\nabla \phi_i(w^{t}) - \nabla \phi_i(w^{*}) \|^2\nonumber \\
	&\leq & 2L^2 \mathbb{E}   \| w^{t} - w^{d(t)}\|^2 + 2 \mathbb{E}   \|\nabla \phi_i(w^{t}) - \nabla \phi_i(w^{*}) \|^2\nonumber \\
	&\leq& 2L^2\eta^2    \mathbb{E}\|\sum\limits_{j=d(t)}^{t-1} v^j\|^2 + 4L\mathbb{E}\bigg( P(w^{t}) - P(w^{*})  - \frac{\lambda}{2} \|w^{t} - w^{*} \|^2 \bigg) \nonumber \\
	&\leq& 2L^2\eta^2 \tau  \sum\limits_{j=d(t)}^{t-1} \mathbb{E}\|v^j\|^2 + 4L\mathbb{E}\bigg( P(w^{t}) - P(w^{*})  - \frac{\lambda}{2} \|w^{t} - w^{*} \|^2 \bigg) 
	%&\leq& 2L^2\eta^2 \tau^2  \mathbb{E}\|v^t\|^2 + 4L\mathbb{E}\bigg( P(w^{t}) - P(w^{*})  - \frac{\lambda}{2} \|w^{t} - w^{*} \|^2 ]\bigg)
	\end{eqnarray}
	where the first inequality follows from Lemma \ref{ex_lem2}, the second inequality follows from Lemma \ref{ex_lem3},  the third and the last  inequalities follow from the Assumption \ref{time_delay}.
	In addition, it also follows that:
	\begin{eqnarray}
	\label{b_111}
	\mathbb{E} [B_{t+1} - B_{t}] \nonumber &=& \mathbb{E}\|w^{t+1} - w^* \|^2 -\mathbb{E}\|w^{t} - w^* \|^2 \nonumber \\
	& = & -2 \eta \mathbb{E} \left< w^{t}-w^*, v^t\right> + \eta^2 \mathbb{E}\|v^t\|^2
	\end{eqnarray}
	We can know $ \mathbb{E} \left< w^{t}-w^*, v^t\right> $ is lower bounded that:
	\begin{eqnarray}
	&&	\mathbb{E} \left< w^{t}-w^*, v^t\right>  \nonumber \\
	&=& \sum\limits_{i \in I_t}	\mathbb{E} \left< w^{d(t)}-w^*, v_i^t\right> + \sum\limits_{i \in I_t}	\mathbb{E} \left< w^{t}-w^{d(t)}, v_i^t\right>  \nonumber \\
	&=&  H \left< w^{d(t)} -w^*,\nabla P(w^{d(t)})\right> +H  \left< w^t - w^{d(t)}, \nabla P(w^{d(t)})\right> \nonumber\\
	&\geq& H\left( P(w^{d(t)}) - P(w^*) \right)+  H\left< w^t - w^{d(t)}, \nabla P(w^{d(t)})\right> 
	\end{eqnarray}
	where the equality follows form that $v_i^t$ is not relevant to the variable before $w^{t+1}$ and the inequality follows from the convexity of $P(w)$. \QEDB \\
\end{proof}

\noindent \textbf{Proof to Theorem \ref{them_convex}}
\begin{proof}
	We define $C_{t+1} = c_aA_{t+1} + c_bB_{t+1}$ and set $c_a=\frac{1}{2\lambda L}$, $c_b=1$.  
	%Therefore, we have the following upper bound of $\mathbb{E} [B_{t+1} - B_{t}] $ that:
	Inputting Lemma \ref{lem2} in the equation, we have:
	\begin{eqnarray}
	\mathbb{E}[C_{t+1}] &=&  c_aA_{t+1} + c_bB_{t+1} \nonumber \\
	&\leq & c_a (1-  \eta \lambda H) \mathbb{E}\|\alpha_i^{t} - \alpha_i^*\|^2 + 2c_a \lambda \tau HL^2\eta^3   \sum\limits_{j=d(t)}^{t-1} \mathbb{E}\|v^j\|^2 -c_a \eta \lambda(1-\beta)  \mathbb{E}\|v^t\|^2   \nonumber\\
	&& + 4c_a \eta \lambda H L\bigg( P(w^{t}) - P(w^{*})  - \frac{\lambda}{2}\mathbb{E} \|w^{t} - w^{*} \|^2 ]\bigg)   +c_b\eta^2 \mathbb{E} \|v^t\|^2 \nonumber \\ 
	&&
	-2c_b\eta \biggl( H \left( P(w^{d(t)}) - P(w^*) \right)+  H\left< w^t - w^{d(t)},\nabla P(w^{d(t)})\right>  \biggr) + c_b \mathbb{E}\|w^t - w^* \|^2 \nonumber \\
	&\leq &(1-\eta \lambda H)  \mathbb{E}[C_t]  + 2\eta  H  \left(P(w^t) - P(w^{d(t)})  - \left<w^t - w^{d(t)}, \nabla P(w^{d(t)}) \right> \right) \nonumber \\
	&&+\left(  \frac{\eta^2 \lambda n}{2L}  + \eta^2  - \frac{\eta}{2L}   \right) \mathbb{E}\|v^t\|^2 +   \tau HL\eta^3   \sum\limits_{j=d(t)}^{t-1} \mathbb{E}\|v^j\|^2  \nonumber \\
	&\leq & (1-\eta \lambda H) \mathbb{E}[C_t]  +\left(  \frac{\eta^2 \lambda n}{2L}  + \eta^2 - \frac{\eta}{2L}   \right) \mathbb{E}\|v^t\|^2 +   2\tau HL\eta^3   \sum\limits_{j=d(t)}^{t-1} \mathbb{E}\|v^j\|^2
	\end{eqnarray}
	where the last inequality follows from the $L$-smooth of $P(w)$: 
	\begin{eqnarray}
	P(w^t) &\leq& P(w^{d(t)}) + \left<w^t - w^{d(t)}, \nabla P(w^{d(t)}) \right> + \frac{L}{2}\|w^t - w^{d(t)}\|^2  \nonumber \\
	&\leq & P(w^{d(t)}) + \left<w^t - w^{d(t)}, \nabla P(w^{d(t)}) \right> + \frac{L\tau \eta^2}{2}\sum\limits_{j=d(t)}^{t-1} \|v^j\|^2 
	\end{eqnarray} 
	Adding the above inequality from $t=0 $ to $t=T-1$, we have that:
	\begin{eqnarray}
	\sum\limits_{t=0}^{T-1} \mathbb{E} [C_{t+1}] &\leq & \sum\limits_{t=0}^{T-1} (1-\eta \lambda H)\mathbb{E}[C_t]  + \left(  \frac{\eta^2 \lambda n}{2L}  + \eta^2 - \frac{\eta}{2L}   \right)  	\sum\limits_{t=0}^{T-1} \mathbb{E}\|v^t\|^2 +   2\tau HL\eta^3   	\sum\limits_{t=0}^{T-1} \sum\limits_{j=d(t)}^{t-1} \mathbb{E}\|v^j\|^2 \nonumber \\
	&\leq &\sum\limits_{t=0}^{T-1} (1-\eta \lambda H)\mathbb{E}[C_t]  + \left(2H\tau^2 \eta^2 + \frac{\eta^2 \lambda n}{2L}  + \eta^2 - \frac{\eta}{2L}   \right)  	\sum\limits_{t=0}^{T-1} \mathbb{E}\|v^t\|^2	
	\end{eqnarray}
	where the last inequality follows from Assumption \ref{time_delay}  and $\eta L \leq 1$. If $2H\eta^2\tau^2 + \frac{\eta^2 \lambda n}{2L}  + \eta^2  - \frac{\eta}{2L}   \leq 0$ such that:
	\begin{eqnarray}
	\eta \leq \frac{1}{4HL\tau^2 + \lambda n + 2L }
	\end{eqnarray}
	Therefore, we have:
	\begin{eqnarray}
	\sum\limits_{t=0}^{T-1}	\mathbb{E} [C_{t+1}] &\leq& \sum\limits_{t=0}^{T-1} (1-\eta \lambda H)  \mathbb{E}[C_t] \nonumber \\
	&\leq & \sum\limits_{t=1}^{T-1}  \mathbb{E}[C_t]  + (1-\eta \lambda H)  C_0
	\end{eqnarray}
	We complete the proof.
	\QEDB	\\
\end{proof}

\noindent \textbf{Proof to Lemma \ref{non_lem2}}
\begin{proof}
	As per the smoothness of $\phi_i$, we have:
	\begin{eqnarray}
	\label{them_iq_1}
	\mathbb{E}\|\nabla \phi_i(w^{d(t)}) + \alpha_i^*\|^2 &=&  \mathbb{E}\|\nabla \phi_i(w^{d(t)})  -\nabla  \phi_i( w^*)\|^2 \nonumber \\
	&\leq& L^2 \mathbb{E}\| w^{d(t)} - w^*\|^2 \nonumber \\
	&\leq & 2 L^2 \mathbb{E}\| w^{d(t)} - w^t\|^2 + 2 L^2 \mathbb{E}\| w^t - w^*\|^2 \nonumber \\
	&\leq & 2 L^2 \eta^2 \tau \sum\limits_{j=d(t)}^{t-1}  \mathbb{E}\| v^j\|^2 + 2 L^2 \mathbb{E}\| w^t - w^*\|^2
	\end{eqnarray}
	We can also bound $-\mathbb{E} \left< w^{t}-w^*, v^t\right> $ as follows:
	\begin{eqnarray}
	\label{them_iq_2}
	-\mathbb{E} \left< w^{t}-w^*, v^t\right>  &= &-H\mathbb{E} \left< w^{t}-w^*, \nabla P(w^{d(t)})\right> \nonumber\\
	&=& -H\mathbb{E} \left< w^{t}-w^*, \nabla P(w^{t})\right> -H \mathbb{E} \left< w^{t}-w^*, \nabla P(w^{d(t)}) - \nabla P(w^{t})\right>\nonumber\\
	&\leq & -\lambda H \mathbb{E} \|w^t - w^*\|^2 + \frac{\gamma H}{2} \mathbb{E}\|w^t - w^*\|^* + \frac{H}{2\gamma} \mathbb{E}\|\nabla P(w^t) - \nabla P(w^{d(t)})\|^2 \nonumber \\
	&\leq & -(\lambda - \frac{\gamma}{2})H \mathbb{E} \|w^t - w^* \|^2 + \frac{H (L+\lambda)^2\eta^2\tau}{2\gamma} \sum\limits_{j=d(t)}^{t-1} \mathbb{E} \| v^j\|^2
	\end{eqnarray}
	where the first inequality follows from the strong convexity of $P$ such that $\left<w^t - w^*, \nabla P(w^t) \right> \geq P(w^t) - P(w^*) + \frac{\lambda}{2} \|w^t - w^*\|^2$ and $P(w^t) - P(w^*) \geq \frac{\lambda}{2}\| w^t - w^*\|^2$.  Defining $\gamma = \frac{\lambda}{2}$ and
	substituting above two inequalities into (\ref{a_111}) and (\ref{b_111}) respectively, we complete the proof.\QEDB \\
\end{proof}

\noindent \textbf{Proof to Theorem \ref{them_nonconvex}}
\begin{proof}
	We define $C_{t+1} = c_aA_{t+1} + c_bB_{t+1}$ and set $c_a=\frac{1}{ 4L^2}$, $c_b=1$.  Inputting Lemma \ref{non_lem2}   in the equation, we have:
	%Setting $\gamma =\frac{\lambda}{2}$ and substituting (\ref{them_iq_1}) and (\ref{them_iq_2}) in (\ref{them_iq_0}), we have:
	
	\begin{eqnarray}
	\mathbb{E}[C_{t+1}] &=&  c_aA_{t+1} + c_bB_{t+1} \nonumber \\
	&\leq &c_a (1-  \eta \lambda H) \mathbb{E}\|\alpha_i^{t}- \alpha_i^*\|^2  + (c_b + 2c_a \eta \lambda HL^2 -2c_b\eta H (\lambda - \frac{\gamma}{2}) ) \mathbb{E}\|w^t - w^*\|^2 \nonumber \\
	&&+ \left( c_b \eta^2 - c_a \eta \lambda (1-\beta)
	\right) \mathbb{E}\|v^t\|^2 + \left(2c_a \lambda H\tau L^2  \eta^3  + c_b \frac{H\tau (L+\lambda)^2\eta^3}{\gamma} \right)\sum\limits_{j=d(t)}^{t-1} \mathbb{E} \| v^j\|^2  \\
	%&\leq & (1- \eta \lambda H) \left(\frac{1}{4L^2}\mathbb{E}\|\alpha_i^{t}x_i - \alpha_i^*x_i\|^2 + \mathbb{E}\|w^t - w^*\|^2 \right) \nonumber \\
	&=& (1- \eta \lambda H) \mathbb{E}[C_t]+ \left( c_b \eta^2 - c_a \eta \lambda (1-\beta)
	\right) \mathbb{E}\|v^t\|^2 + \left(2c_a \lambda H\tau L^2  \eta^3  + c_b \frac{H\tau (L+\lambda)^2\eta^3}{\gamma} \right)\sum\limits_{j=d(t)}^{t-1} \mathbb{E} \| v^j\|^2\nonumber 
	\end{eqnarray}
	where $\gamma = \frac{\lambda}{2}$.
	Adding the above inequality from $t=0 $ to $t=T-1$, we have that:
	\begin{eqnarray}
	\sum\limits_{t=0}^{T-1} \mathbb{E}[C_{t+1}] &\leq & (1-\eta \lambda H) 	\sum\limits_{t=0}^{T-1} \mathbb{E}[C_{t}]   + \left( c_b \eta^2 - c_a \eta \lambda (1-\beta)
	\right)  	\sum\limits_{t=0}^{T-1} \mathbb{E}\|v^t\|^2  \nonumber \\
	&& + \left(2c_a \lambda H\tau L^2  \eta^3  + c_b \frac{H\tau (L+\lambda)^2\eta^3}{\gamma} \right) 	\sum\limits_{t=0}^{T-1}\sum\limits_{j=d(t)}^{t-1} \mathbb{E} \| v^j\|^2 \nonumber \\
	&\leq & (1-\eta \lambda H) 	\sum\limits_{t=0}^{T-1} \mathbb{E}[C_{t}]  +  \left( c_b \eta^2 - c_a \eta \lambda (1-\beta) + 2c_a \lambda H\tau^2 L^2  \eta^3  + c_b \frac{H\tau^2 (L +\lambda)^2\eta^3}{\gamma}
	\right)  	\sum\limits_{t=0}^{T-1} \mathbb{E}\|v^t\|^2  \nonumber \\
	&\leq & 	\sum\limits_{t=1}^{T-1} \mathbb{E}[C_{t}]   +  (1-\eta \lambda H) \mathbb{E}[C_0]
	\end{eqnarray}	
	where the last inequality holds as long as:
	\begin{eqnarray}
	\eta \leq \frac{\lambda^2}{2HL\tau^2\lambda^2 + 8 HL \tau^2  (L+\lambda)^2 + 4\lambda L^2 + n \lambda^3 }
	\end{eqnarray} 
	such that $2c_a \lambda HL^2 \tau^2 \eta^3  + c_b \frac{H(L+\lambda)^2\tau^2\eta^3}{\gamma} + c_b \eta^2 - c_a \eta \lambda (1-\beta) \leq 0$. 	
	\iffalse
	We consider outer iteration $s$ right now, and write $C_T$ as $C_{s, T}$.  Because $C_{s, T} = C_{s+1,0}$, we have the following inequality:
	\begin{eqnarray}
	\mathbb{E}[C_{s+1,0}] &\leq & (1-\eta \lambda H)C_{s,0}
	\end{eqnarray} 
	Applying the algorithm for $S$ iterations, we have:
	\begin{eqnarray}
	\mathbb{E}[C_{S, 0}] \leq (1-\eta \lambda H)^S C_{0,0}
	\end{eqnarray}
	\fi
	We complete the proof. 
	\QEDB\\
\end{proof}

\section{Extra Lemmas}

\begin{lemma}[\cite{reddi2016fast}]
	For random variables $z_1, . . . , z_r$ are independent and mean 0, we have:
	\begin{eqnarray}
	\mathbb{E} [ \| z_1 + ... + z_r\|^2 ]  &=& \mathbb{E} [\|z_1\|^2 + ...+ \|z_r\|^2 ]
	\end{eqnarray}
	\label{ex_lem1}
\end{lemma}

\begin{lemma}
	For any $z_1,...,z_r$, it holds that:
	\begin{eqnarray}
	\| z_1 + ... + z_r\|^2 &  \leq  & r (\|z_1\|^2 +...+ \|z_r\|^2 )
	\end{eqnarray}
	\label{ex_lem2}
\end{lemma}

\begin{lemma}[\cite{shalev2015sdca}]
	Assume that each $\phi_i(w)$ is $L$-smooth and convex. Then, for every $w$,
	\begin{eqnarray}
	\frac{1}{n}\sum\limits_{i=1}^n \|\nabla \phi_i(w) - \nabla \phi_i(w^*)\|^2 &\leq& 2L\left( P(w) - P(w^{*})  - \frac{\lambda}{2} \|w - w^{*} \|^2 \right) 
	\end{eqnarray}
	\label{ex_lem3}
\end{lemma}

\end{document}